\documentclass{article}

\usepackage[preprint]{neurips_2026}

\usepackage[utf8]{inputenc} 
\usepackage[T1]{fontenc}    
\usepackage{hyperref}       
\usepackage{url}            
\usepackage{booktabs}       
\usepackage{multirow}
\usepackage{amsfonts}       
\usepackage{nicefrac}       
\usepackage{microtype}      
\usepackage{xcolor}         
\usepackage{subcaption}

\usepackage{times} 
\usepackage{mymacro}
\usepackage{mycommand}
\usepackage{enumitem} 

\allowdisplaybreaks[3]

\usepackage{natbib}
\setcitestyle{authoryear,open={(},close={)}}
\title{Escaping Iterative Parameter-Space Noise: 
Differentially Private Learning with a Hypernetwork}

\author{%
  Naoki Nishikawa\\
  The University of Tokyo\thanks{This work was done while the author was at LY Corporation.}\\
  \texttt{nishikawa-naoki259@g.ecc.u-tokyo.ac.jp} \\
  \AND
  Shokichi Takakura\\
  LY Corporation\\
  \texttt{stakakur@lycorp.co.jp} \\
  \And
  Satoshi Hasegawa \\
  LY Corporation\\
  \texttt{satoshi.hasegawa@lycorp.co.jp} \\
}

\begin{document}

\maketitle

\begin{abstract}
    Differentially private (DP) training of neural networks is often hindered by the large amount of noise required by gradient-based methods such as DP-SGD, which repeatedly inject high-dimensional noise in parameter space throughout training. 
    In this paper, we propose a new framework for DP learning that avoids iterative optimization in parameter space. 
    Instead of updating the target model using privatized gradients, 
    we employ a \emph{hypernetwork trained on public datasets}
    to map a private dataset to the parameters of the target model. 
    Specifically, each example is embedded into a low-dimensional representation, the embeddings are aggregated and perturbed to obtain a DP dataset embedding, and the hypernetwork generates the target model parameters from this noisy embedding.
    Because privacy noise is injected only once into a low-dimensional dataset representation, our approach can significantly reduce the adverse effect of noise.
    We theoretically show in a synthetic setting that, under a fixed privacy budget, models produced by our approach achieve higher utility than those trained with DP-SGD. 
    Moreover, we apply our approach to LoRA fine-tuning of diffusion models and show that it achieves lower FID than LoRA models trained with DP-SGD and other public-data-guided methods.
\end{abstract}

\section{Introduction}

Deep neural networks are often trained on data that include private information.
Since training data can sometimes be reconstructed from the parameters of trained models~\citep{carlini2023extracting,wu2022membership}, it is necessary to prevent the privacy of training data from being leaked through the parameters of neural networks.
One approach is to ensure that the training algorithm for neural networks satisfies \emph{differential privacy (DP)}, a mathematical framework for quantifying privacy leakage from a dataset.
The most widely used approach is DP-SGD~\citep{abadi2016deep}, which protects the privacy of each training example by adding noise at each step of SGD.
Several other gradient-based methods have also been proposed, such as DP-Adam~\citep{zhou2020private} and DP-AdamW~\citep{sun2025dp}.

Training neural networks with these gradient-based DP learning methods can incur a substantial utility loss compared with non-private optimization methods~\citep{tramer2021differentially}. 
A major source of this degradation is the amount of noise required for privacy. 
In gradient-based methods, DP noise is repeatedly injected at every optimization step. 
Moreover, DP noise must be added in the gradient space, whose dimension is equal to the number of model parameters.
Even if the gradient norm is fixed, the scale of the required DP noise grows at least on the order of the square root of the dimension and the number of steps. 
This issue becomes particularly severe in small-data regimes, such as when finetuning a large model on a limited amount of private data.
These observations naturally lead to the following question:
\begin{center}
    \emph{Can we obtain high-utility models under differential privacy without repeatedly} \\
    \emph{injecting parameter-dimensional noise?}
\end{center}

As a positive answer to this question, we propose a new framework that directly generates differentially private parameters using a hypernetwork.
Instead of updating parameters using gradients computed from a private dataset, we feed the private dataset into a hypernetwork and obtain parameters adapted to the dataset through a single forward pass.
To this end, we propose \emph{DP-DeepSets}, a hypernetwork architecture for DP learning that takes a dataset as input and outputs parameters with DP guarantees.
This architecture first embeds each data point into a low-dimensional representation, then applies mean pooling and adds DP noise to obtain a differentially private dataset embedding.
The parameters are generated from this embedding, and the resulting network inherits the DP guarantee.

Our method has three key advantages by design.
First, \emph{DP perturbation occurs only once during parameter generation}.
Unlike gradient-based methods, which inject noise at every step, our method adds DP noise only once, to the dataset embedding.
Second, \emph{DP noise is injected into a low-dimensional dataset embedding}.
In contrast to gradient-based methods, which add noise to high-dimensional gradients in parameter space, our approach requires a much smaller noise scale.
Consequently, the performance of the learned model is less susceptible to degradation caused by DP noise.
Third, \emph{the hypernetwork itself is trained on public data}.
Since the hypernetwork is public, the privacy budget is consumed only when constructing the low-dimensional dataset embedding during parameter generation from private data.
Moreover, training the hypernetwork incurs no privacy cost, allowing it to learn the relationship between datasets and parameters from a large amount of public data, and thereby acquire a form of ``learning to learn.''

Our contributions are summarized as follows:
\begin{enumerate}[leftmargin=5mm]
    \item We propose a novel framework for DP learning of neural networks using public hypernetworks.
    Specifically, we develop \emph{DP-DeepSets}, which takes a dataset as input and returns parameters with a DP guarantee.
    \item We theoretically compare DP learning using DP-DeepSets with DP-SGD.
    We consider a synthetic setting involving linear regression, and show that DP-DeepSets achieves high utility since it learns to capture the low-dimensional structure of tasks during training on public data.
    \item We apply our framework to LoRA finetuning of diffusion models.
    We demonstrate that the model trained by our method outperforms DP-SGD and other public-data-guided DP learning methods.
\end{enumerate}

\begin{figure}[t]
  \centering
  \hspace{3mm}
  \begin{minipage}{0.4\textwidth}
    \centering
    \includegraphics[width=0.8\linewidth]{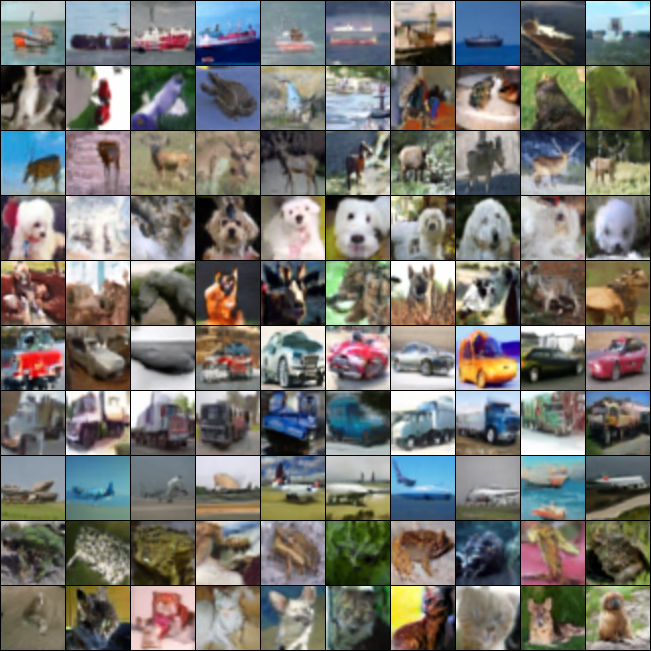}
    \captionof{figure}{Generated samples on CIFAR-10 with 
        $(\epsilon, \delta)=(2, 10^{-5})$.
        Each row corresponds to one $k$-NN set.
    }
    \label{fig:generated-images}
  \end{minipage}
  \hfill
  \begin{minipage}{0.5\textwidth}
    \centering
    \includegraphics[width=0.95\linewidth]{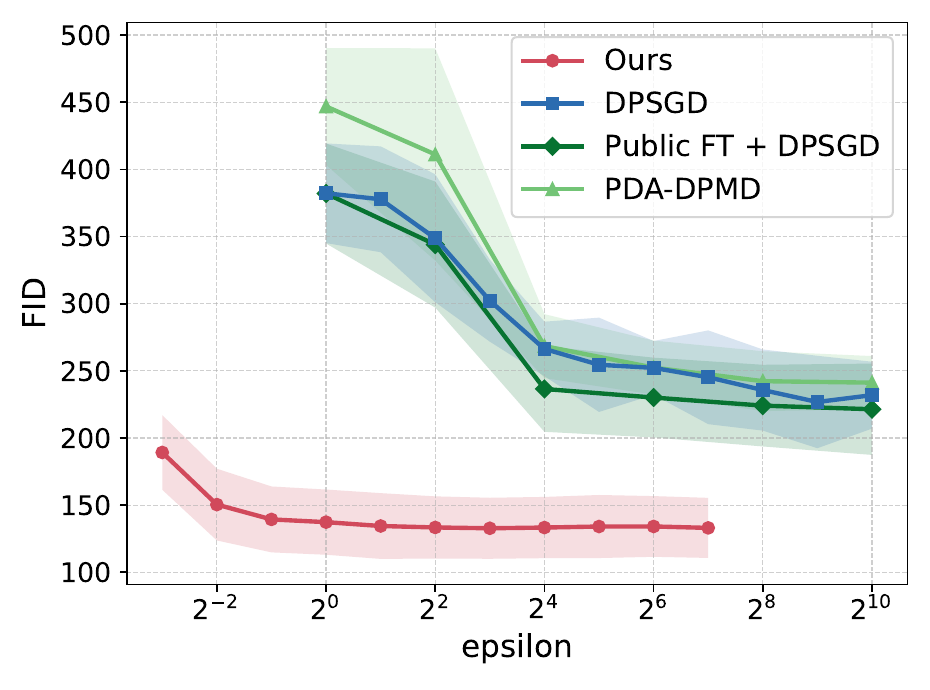}
    \vspace{-2mm}
    \captionof{figure}{Change in FID with respect to the DP parameter $\epsilon$ for each method.}
    \label{fig:fid-for-epsilons}
  \end{minipage}
  \hspace{3mm}
\end{figure}

We present results for LoRA finetuning of diffusion models in Figures \ref{fig:generated-images} and \ref{fig:fid-for-epsilons}.
The images in \cref{fig:generated-images} are generated by applying LoRA parameters obtained from only 128 private data points via the trained hypernetwork.
The generated images capture the characteristics of each CIFAR-10 class.
Furthermore, \cref{fig:fid-for-epsilons} shows that our method outperforms DP-SGD and other public-data-guided methods.
See \cref{sec:experiments} for detailed settings and additional observations.

\paragraph{Related work}
The concept of hypernetworks was introduced by \citet{ha2017hypernetworks}, who proposed using one RNN to determine the parameters of another. 
This idea was later extended to neural network parameter generation by \citet{peebles2022learning} and \citet{wang2024neural}. 
To enable task-conditioned parameter generation, several methods have incorporated task information into prompts for model-parameter generation \citep{jin2024conditional,wang2025recurrent,khan2025oral}. 
Unlike these prompt-based approaches, we explore a dataset-conditioned approach that aims to replace gradient-based DP learning with a hypernetwork. 
Recent studies have also proposed dataset-conditioned parameter generation \citep{ruiz2024hyperdreambooth,soro2025diffusion,liang2025drag}; 
building on this line of work, our method introduces an architecture tailored to parameter generation with a DP guarantee.

Several prior studies have leveraged public data to improve model performance while preserving privacy.
In particular, AdaDPS~\citep{li2022private} and PDA-DPMD~\citep{amid2022public} incorporate public-data information into gradient-based methods such as Adam and mirror descent.
They still rely on gradient-based optimization and thus cannot avoid adding noise in a space whose dimension equals the number of model parameters.
Another line of work leveraging public data includes PATE~\citep{papernot2016semi} and its variants~\citep{papernot2018scalable}.
They use private data to train models that label large amounts of unlabeled public data.
This line of work differs from our framework in that it focuses on classification tasks and explores the use of unlabeled public data.

Among public-data-guided methods, the closest to ours is \citet{raisa2024noise}, which generates parameters by feeding a private dataset into a network called SetConv trained on simulated data.
Although the idea is related, 
their experiments are limited to regression with one-dimensional inputs, whereas we address high-dimensional tasks.
This difference motivates a different hypernetwork design:
SetConv applies a fixed kernel function to each data point and aggregates the resulting representations, but such fixed-kernel data encoding does not work well for high-dimensional data.
Our method instead uses a DeepSets-based architecture with a learnable feature map to embed data points, 
allowing the hypernetwork to learn how to efficiently encode high-dimensional data.

Several studies, similar to ours, focus on the high dimensionality of gradients and seek to learn in a low-dimensional space.
For example, \citet{yu2021not} and \citet{zhou2021bypassing} project gradients onto a subspace using public-data information, while \citet{zheng2026differentially} use training-trajectory information for subspace projection.
In addition, \citet{zhu2023improving} sparsify gradients using random masks.
While these approaches reduce noise by lowering gradient dimensionality, they still rely on gradient-based optimization and thus cannot avoid repeatedly adding noise during training.

\section{Background}

\subsection{Differential Privacy}

Differential Privacy~(DP) \citep{dwork2006calibrating} is a mathematical framework for quantifying and limiting the information about any individual data point that can be inferred from the output of an algorithm. 
Let $\scD^\ast$ denote the space of datasets.
Two datasets $D, D' \in \scD^\ast$ are said to be adjacent 
if they differ in a single entry.
We say that a randomized algorithm $\scM: \scD^\ast \to \scY$ satisfies $(\epsilon, \delta)$-DP if, for any pair of adjacent datasets $D, D' \in \scD^\ast$ and any measurable set $S \subseteq \scY$, it holds that
\begin{align*}
    \bP(\scM(D) \in S) \leq \rme^{\epsilon} \cdot \bP(\scM(D') \in S) + \delta.
\end{align*}
In other words, $(\epsilon, \delta)$-DP requires that the output distribution does not change substantially when a single data point is replaced.

For a query that takes a dataset as input and outputs a real-valued quantity, 
a common way to guarantee differential privacy is to add Gaussian noise to the output.
This is called the Gaussian mechanism, and it provides the following differential privacy guarantee.
\begin{proposition}[Gaussian mechanism, \citet{dwork2014algorithmic}]\label{prop:DP-gaussian-mechanism}
    For a given query $q: \scD^\ast \to \bR^d$,
    we define the $s^2$-\emph{Gaussian mechanism} $\scM_q: \scD^\ast \to \bR^d$ as
    $
        \scM_q(D) := z + q(D)
    $
    where $z \sim \scN(0, s^2 I_d)$.
    Let $\Delta := \max_{D\sim D'}\|q(D) - q(D')\|_2$.
    Then, for any $\epsilon, \delta \in (0, 1)$, 
    the $s^2$-Gaussian mechanism with
    $s^2 = 2\log(1.25\delta^{-1})(\Delta/\epsilon)^2$ 
    satisfies $(\epsilon, \delta)$-differential privacy.
\end{proposition}

\subsection{Differentially Private SGD~(DP-SGD)}\label{sec:dp-sgd-background}

For training a machine learning model on a private dataset, 
one of the most widely used algorithms for ensuring DP of a trained model is Differentially Private SGD~\citep{abadi2016deep}.
In this algorithm, Gaussian noise is added when computing the average of the gradients in a minibatch.
More specifically, at each step, DP-SGD randomly selects $L$ data points $p_{n_1}, \ldots, p_{n_L}$ from a private dataset $p_{1:N}$ consisting of $N$ data points, computes the gradients $g_{n_1}(\omega), \ldots, g_{n_L}(\omega)$ with respect to the current parameter $\omega$ for each selected data point, and then updates the parameters as follows:
\begin{align*}
    \omega
    \leftarrow \omega
        - \frac{\eta}{L} \qty(
            \sum_{i=1}^L \clip_C(g_{n_i}(\omega))
            + \xi
        ),
    \quad \xi \sim \scN(0, s^2 I_{\Dparam})
\end{align*}
where $
    \clip_C(x) = \min\qty{1, \frac{C}{\|x\|}} x
$, 
$\eta > 0$ is a learning rate and $\Dparam$ is the dimension of the parameter $\omega$.
DP-SGD satisfies the following differential privacy guarantee.
\begin{proposition}[Theorem 1 of \cite{abadi2016deep}]
    There exist constants $c_1, c_2$ such that, 
    given the number of steps $T$, 
    for any $\epsilon < c_1 \frac{L^2 T}{N^2}$, 
    DP-SGD is $(\epsilon, \delta)$-differentially private for any $\delta>0$,
    if we choose $s^2 \geq c_2 \frac{C^2 L^2 T\log\delta^{-1}}{\epsilon^2 N^2}$.
\end{proposition}

\section{Framework: Differentially Private Learning with a Public Hypernetwork}\label{sec:framework}

In this section, we formally introduce our framework for learning neural networks with a DP guarantee using a hypernetwork.
Consider the setting in which we solve a given task using a neural network $\psi_\omega$ parameterized by $\omega \in \bR^{\Dparam}$.
We refer to this network as the \emph{target network}.
Let $p \in \bR^{\Ddata}$ denote a data point, and let $u \in \bR^\Daux$ denote auxiliary variables used to define the network input and the training loss.
For example, in a regression task, $p$ is the pair consisting of an input $x$ and its label $y$, and $u$ may be omitted.
In generative modeling, such as autoregressive language models or diffusion models, $p$ corresponds to the sample itself, whereas $u$ contains the auxiliary inputs needed to evaluate the loss.
For instance, in an autoregressive language model, $u$ may specify the prediction position or context prefix, while in a diffusion model, $u$ may include the noise and timestep.

Let $\ell : \bR^{\Ddata} \times \bR^{\Daux} \times \bR^{\Dparam} \to \bR$ be a per-data-point loss function that takes a data point, auxiliary variables, and the network parameters as input, and returns the loss associated with that data point.
For example, in a regression task with $p = (x,y)$, one may write
$\ell(p, u; \omega) = \frac{1}{2}(\psi_\omega(x) - y)^2$,
where $u$ is absent or trivial,
whereas in image generation with diffusion models, 
$\ell$ is typically the mean-squared error for noise prediction.

Our goal is to obtain a hypernetwork $\phi$ that takes a dataset $p_{1:N}$ and DP noise $\xi$ as input and generates parameters that satisfy an $(\epsilon, \delta)$-DP guarantee and achieves a small test loss
$\bE_{p,u,\xi}[\ell(p, u; \phi(p_{1:N}, \xi))]$.
To this end, we develop \emph{DP-DeepSets}, a hypernetwork architecture that takes the dataset as input and predicts parameters with a DP guarantee.
Moreover, we establish a training procedure for DP-DeepSets using a public dataset.

\subsection{DP-DeepSets: A Hypernetwork Architecture for Differentially Private Learning}

\begin{figure}[t]
    \centering
    \includegraphics[width=0.98\linewidth]{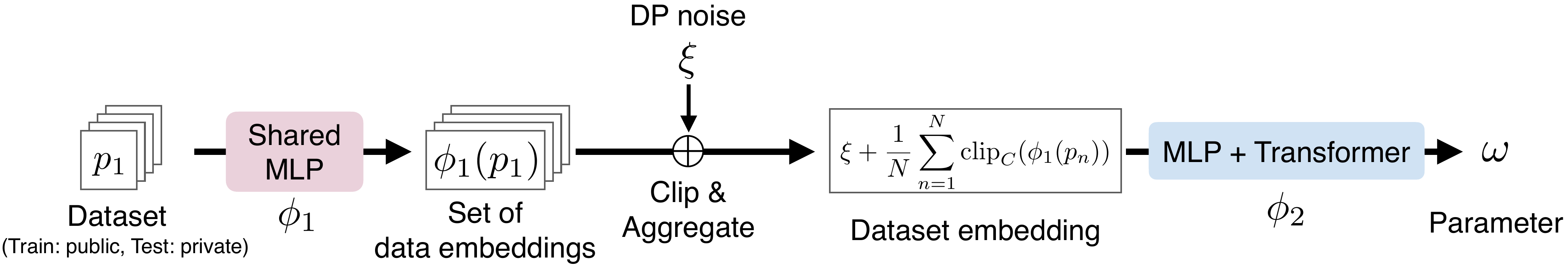}
    \captionof{figure}{
        Overview of DP-DeepSets.
        First, we obtain embeddings of the data points in the input dataset using an MLP $\phi_1$, and clip them in $L^2$-norm.
        Then, we take the element-wise average and perturb it with Gaussian noise.
        As a result, we obtain a dataset embedding with a DP guarantee, 
        and feed it into the latter network $\phi_2$ to generate the network parameter.
    }
    \label{fig:dp-deepsets}
\end{figure}

Here, we describe the mechanism of \emph{DP-DeepSets}.
Its overview is also shown in \cref{fig:dp-deepsets}.

DP-DeepSets takes a dataset, i.e., a set of data points $p_{1:N} := \set{p_n}_{n=1}^N$ as input.
It first obtains a set of data embeddings $\set{\phi_1(p_n)}_{n=1}^N \subset \bR^D$ using an MLP $\phi_1: \bR^{\Ddata}\to\bR^{D}$ shared among the data points.
Next, each embedding is clipped to ensure that its norm is at most $C > 0$.
This procedure is required to limit privacy leakage from each data point.
Then, to aggregate the information in the entire dataset, we take the average of the clipped embeddings.
Finally, we add Gaussian noise $\xi \sim \scN(0, \sdp^2 I_D)$ for some $\sdp$ and guarantee that the resulting embedding is differentially private.
As a result, we obtain the following noisy dataset embedding:
\begin{align*}
    E(p_{1:N}, \xi) 
    := \xi + \frac{1}{N} \sum_{n=1}^N \clip_C(\phi_1(p_n)).
\end{align*}

\begin{figure}[t]
    \centering
    \includegraphics[width=0.85\linewidth]{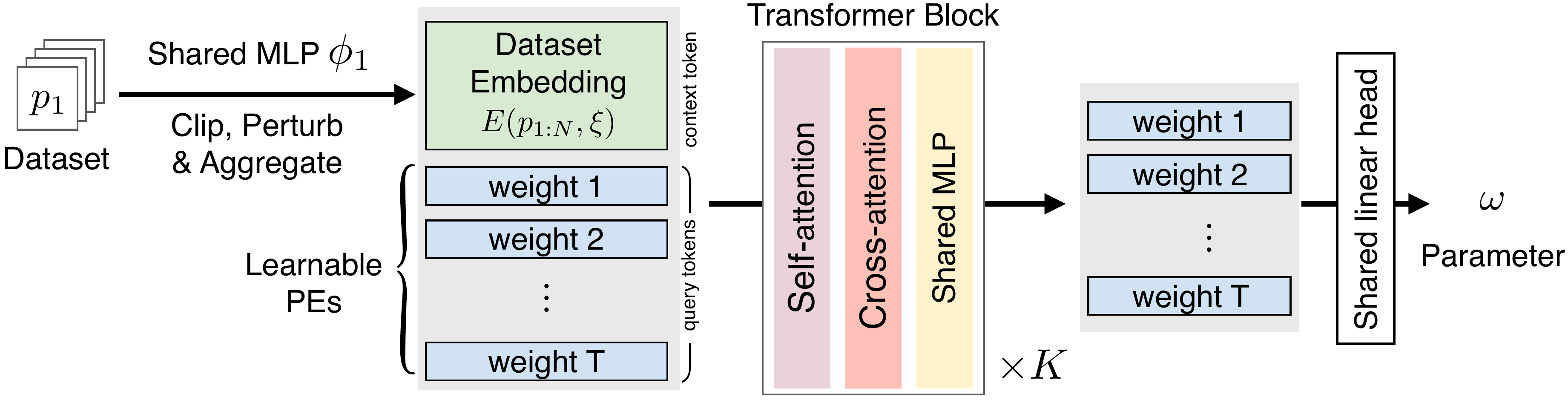}
    \captionof{figure}{
        The architecture of the transformer-based decoder $\phi_2$ in DP-DeepSets.
        The dataset embedding is used as a context token.
        The parameters are divided into $T$ query tokens, and are initialized with learnable positional encodings.
        They are fed into the $K$-layer Transformer consisting of self-attention, cross-attention and shared MLPs.
        The output tokens are fed into the shared linear head, 
        and are reshaped to integrate into the model parameters.
    }
    \label{fig:tf-based-decoder}
\end{figure}

Using a dataset embedding $E(p_{1:N}, \xi)$, we can generate the network parameters as $\omega = \phi_2(E(p_{1:N}, \xi))$ with some neural network $\phi_2: \bR^D \to \bR^{\Dparam}$.
Inspired by \citet{ruiz2024hyperdreambooth}, we use a Transformer-based architecture for $\phi_2$ as in \cref{fig:tf-based-decoder}.
Specifically, the dataset embedding is fed into a transformer as a context token.
The parameters to be generated are divided into $T$ tokens, and they are initialized as learnable positional encodings.
Then, they are iteratively updated through multiple transformer blocks (namely, self-attention among the query tokens, cross-attention between the query tokens and the context token, and token-wise shared MLPs).
This allows the hypernetwork not merely to generate the parameters at each position independently of the dataset embedding, but also to model the interactions among parameters in different layers.

The architecture above can be seen as a \emph{differentially private variant} of DeepSets~\citep{zaheer2017deep}.
DeepSets uses operations such as taking the maximum, summation, and average of the embeddings obtained by a neural network to aggregate the information of sets.
Building upon this idea, we adopt the average operation, and we further incorporate the clipping and Gaussian perturbation to guarantee DP. 
This allows us to apply the theoretical guarantee for the Gaussian mechanism for the mean query, and obtain a rigorous DP guarantee as follows:
\begin{corollary}\label{cor:gaussian-mechanism-for-mean}
    Let $\epsilon, \delta \in(0, 1)$.
    If $\sdp^2 \geq \frac{8C^2 \log(1.25\delta^{-1})}{\epsilon^2 N^2} $, 
    then DP-DeepSets is $(\epsilon, \delta)$-differentially private.
\end{corollary}
The proof can be found in \cref{app:proof-hypernet-dp}.

\subsection{Training of a Public Hypernetwork}\label{subsec:training-hypernet}

\begin{figure}[t]
    \centering
    \includegraphics[width=0.65\linewidth]{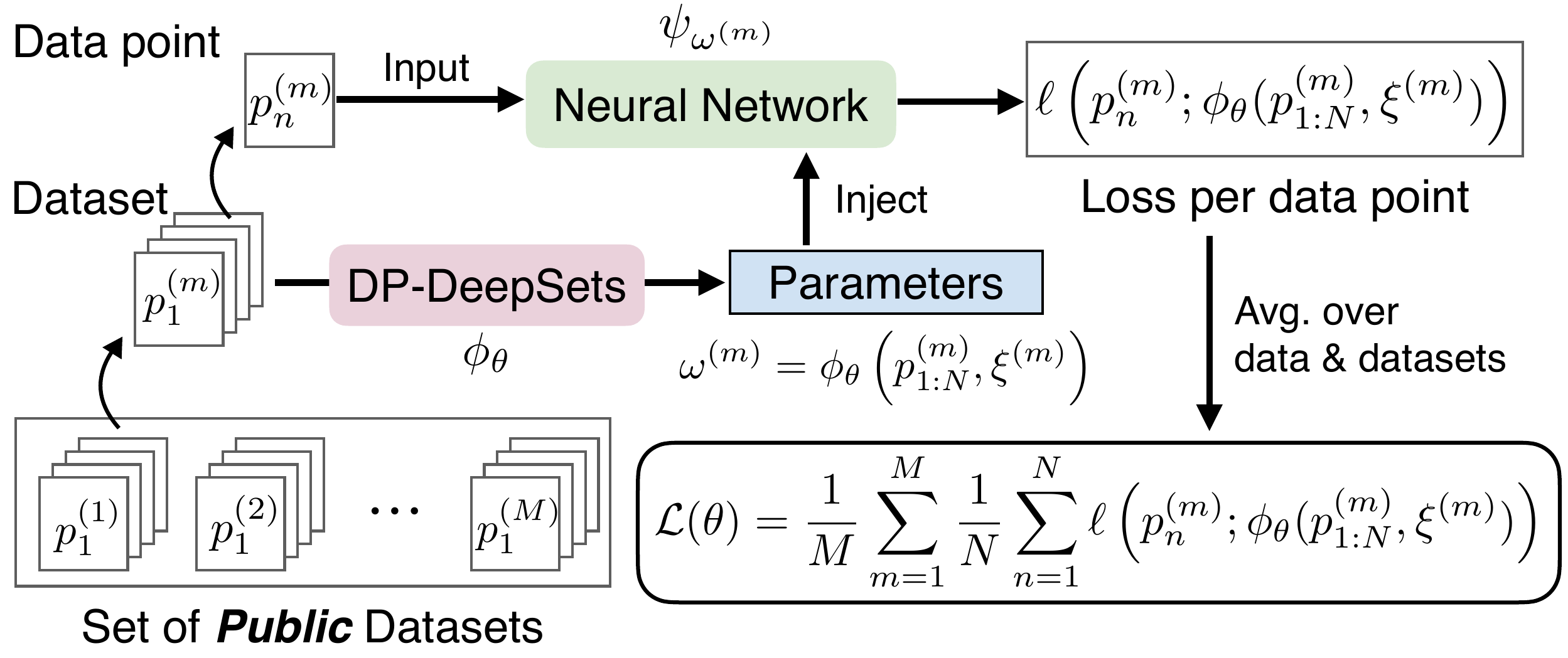}
    \captionof{figure}{
        The training procedure for the DP-DeepSets hypernetwork.
        To train the hypernetwork parameter $\theta$, we prepare a collection of \emph{public} datasets.
        The parameter $\theta$ is trained by minimizing the average dataset-wise loss over all public datasets.
        For each dataset $\pidx{m}_{1:N}$, DP-DeepSets $\phi_\theta$ generates
        $\omegaidx{m} := \phi_\theta(\pidx{m}_{1:N}, \xiidx{m})$.
        The generated parameter $\omegaidx{m}$ is injected into the target network, and the dataset-wise loss is computed as the average per-data-point loss over the data points in the dataset.
    }
    \label{fig:hn-train-procedure}
\end{figure}

Next, we describe how to train DP-DeepSets.
The procedure is also described in \cref{fig:hn-train-procedure}.
To use a hypernetwork to generate suitable parameters for private data, 
we need to train it to learn the correspondence between datasets and the parameters beforehand.
For this purpose, we propose to \emph{train DP-DeepSets using a set of public datasets}.

Specifically, to train DP-DeepSets, we prepare a collection of public datasets
$\set{(\pidx{1}_{1:N}, \uidx{1}_{1:N}), \ldots, (\pidx{M}_{1:N}, \uidx{M}_{1:N})}$.
We also independently sample noise vectors
$\xiidx{1}, \ldots, \xiidx{M} \sim \scN(0, \sdp^2 I_D)$.
Let $\phi_\theta$ denote a DP-DeepSets hypernetwork parameterized by $\theta$.
We train it using the following loss function $\scL_{\HN}(\theta)$:
\begin{align}\label{eq:hypernet-loss}
    \!\scL_{\HN}(\theta)
    \!:=\!\frac{1}{M} \sum_{m=1}^M \!\scLidx{m}_{\dataset}(\theta), 
    ~\text{where}~
    \scLidx{m}_{\dataset}(\theta)
    \!:=\! \frac{1}{N} \sum_{n=1}^N 
        \!\ell\qty(\pidx{m}_n, \uidx{m}_n; \phi_\theta(\pidx{m}_{1:N}, \xiidx{m})\!).
\end{align}
In short, the hypernetwork is trained end-to-end to minimize the average loss over the $M$ public datasets.
For each dataset, the loss $\scLidx{m}_{\dataset}$ is essentially the standard empirical risk,
obtained by averaging the per-sample loss over the $N$ samples.
The only difference is that the network parameter $\omegaidx{m}$ is not a directly optimized parameter;
instead, it is generated by the hypernetwork as
$\omegaidx{m} = \phi_\theta(\pidx{m}_{1:N}, \xiidx{m})$.
By minimizing this end-to-end objective over dataset-specific losses,
the hypernetwork directly learns what parameters it should generate to reduce the loss for each dataset.

\begin{remark}
    Since DP-DeepSets is trained using public data, \emph{it does not contain any private information}.
    Therefore, during inference on private data, we do not care about privacy leakage from DP-DeepSets itself.
    The only thing that matters for privacy leakage is the parameters generated from the private data through DP-DeepSets.
    Despite this, \emph{to avoid distribution shift, we perturb the dataset embedding even during the training phase, which uses public datasets}.
\end{remark}

\vspace{-2mm}
\section{Theoretical Analysis: A Case Study in Linear Regression}\label{sec:theory}
\vspace{-2mm}

In this section, we theoretically analyze DP-DeepSets and DP-SGD in a synthetic setting.
We show that DP-DeepSets effectively encodes datasets in low-dimensional embeddings and has high utility compared to DP-SGD.

We consider linear regression.
Suppose that the input vector $x$ is drawn from $\unif(\sqrt{d}~\bS^{d-1})$.
Given a true parameter $w \in \bR^{d}$, the observation $y \in \bR$ is given as $y = w^\top x + z$, where $z \sim N(0, 1)$ is observation noise independent of $x$.

We assume that the true parameter is randomly drawn as $w\sim\unif(\set{v \mid \|A_w^{-1} v\|_\infty \leq 1})$, where $A_w := \diag(\alpha_1, \ldots, \alpha_d)$ for $\alpha_1, \ldots, \alpha_d > 0$.
Let $x_1, \ldots, x_N \sim \unif(\sqrt{d}~\bS^{d-1})$ be i.i.d. inputs, 
and $y_n := w^\top x_n + z_n$~($n=1, \ldots, N$) be observations with i.i.d. noise $z_1, \ldots, z_N \sim \scN(0, 1)$.
Our goal is to estimate the true parameter $w$ 
given a dataset $p_{1:N} = \{(x_1, y_1), \ldots, (x_N, y_N)\}$.

We also make the following assumption on the distribution of $w$.
\begin{assumption}\label{assmp:eigen-of-distr-w}
    It holds $\alphaul i^{-\gamma} \leq \alpha_i \leq \alphabar i^{-\gamma}$ 
    for some constants $\alphaul, \alphabar > 0$ and $\gamma > 1/2$.
\end{assumption}
This implies that the coordinates with larger indices are progressively less important, in the sense that the allowable magnitude of $w_i$ decays polynomially with $i$.
Therefore, the true parameter $w$ has a small effective dimension, and accordingly the information needed to recover $w$ from the dataset $p_{1:N}$ is also low-dimensional. 
If an algorithm can successfully exploit this low-dimensional structure, then it can avoid large DP noise and estimate the parameter efficiently.

\subsection{Upper Bound of Estimation Error for DP-DeepSets}

First, we provide an upper bound on the error in estimating $w$ using DP-DeepSets.
For simplicity, we consider DP-DeepSets consisting only of MLPs without Transformers.

To rigorously define the class of DP-DeepSets, we introduce the class of MLPs.
Let $\eta := \max\{0, \cdot\}$ be the ReLU activation function.
Then, an MLP $f: \bR^{\Din} \to \bR^{\Dout}$ with depth $L$ and width $W$ is defined as
$
    f(x) = (A_L \eta(\cdot) + b_L)  \circ \cdots \circ (A_2 \eta(\cdot) + b_2)  
            \circ (A_1 x + b_1), 
$
where $A_i \in \bR^{d_{i+1} \times d_i}$, $b_i \in \bR^{d_{i+1}}$
with $d_1 = \Din$, $d_{L+1} = \Dout$ and $\max_i d_i \leq W$.
Then, we define the class of MLPs with input dimension $\Din$, output dimension $\Dout$, depth $L$, width $W$, and norm constraint $B$ as
\begin{align*}
    \Phi^{\Din, \Dout}_\mlp(L, W, B)
    := \left\{
        f
        \relmiddle|
        \max_i \{\norm{A_i}_\infty \vee \norm{b_i}_\infty\} \leq B
    \right\}.
\end{align*}
Next, we define the class of DP-DeepSets.
DP-DeepSets $\phi$ (using MLPs instead of a Transformer)
takes a dataset $p_{1:N}$ and DP noise $\xi$ as inputs, 
and predicts the parameter as follows:
\begin{align*}
    \phi(p_{1:N}, \xi)
    := \clip_F \circ \phi_2\qty(
        \xi + \frac{1}{N} \sum_{n=1}^N \clip_C \circ \phi_1(p_n), 
    )
\end{align*}
where $\phi_1: \bR^{\Ddata} \to \bR^D$ and $\phi_2: \bR^D \to \bR^{\Dparam}$
are MLPs, and $F > 0$ is a fixed constant.
Here, we add a clipping operation compared to the architecture introduced in \cref{sec:framework} to constrain expressivity.
Such clipping operations are commonly used in prior work analyzing the estimation error of neural networks~(e.g., \citet{suzuki2018adaptivity}, \citet{schmidt2020nonparametric}).
Then, we define $\Phi(L, W, B, C, D)$ as the class of DP-DeepSets 
with MLP depth $L$, width $W$, norm constraint $B$, clipping size $C$, and embedding dimension $D$:
\begin{align*}
    \Phi(L, W, B, C, D)
    = \left\{
        \phi \relmiddle| 
        \phi_1 \in \Phi_\mlp^{\Ddata, D}(L, W, B), ~
        \phi_2 \in \Phi_\mlp^{D, \Dparam}(L, W, B)
    \right\}.
\end{align*}

As described in \cref{subsec:training-hypernet}, we employ public datasets 
$\pidx{1}_{1:N}, \ldots, \pidx{M}_{1:N}$ and DP noise vectors $\xiidx{1}, \ldots, \xiidx{M}$ to train the hypernetwork.
We analyze the hypernetwork $\phihat$ chosen to minimize the loss $\scL_\HN$ defined in \eqref{eq:hypernet-loss} after setting $\ell$ to the mean squared error, i.e.,
\begin{align}\label{eq:definition-phihat}
    \phihat := \argmin_{\phi\in\Phi(L, W, B, C, D)}
    \frac{1}{M} \sum_{m=1}^M
    \frac{1}{N} \sum_{n=1}^N 
    \qty(\phi(\pidx{m}_{1:N}, \xiidx{m})^\top \xidx{m}_n - \yidx{m}_n)^2
\end{align}

Then, the parameters generated by the trained DP-DeepSets $\phihat$ for a test (private) dataset satisfy the following theorem.
The proof can be found in \cref{app:proof-of-upper-bound-hn}.
\begin{theorem}\label{thm:upper-bound-hyper-network}
    Let $k \in \{1, \ldots, d\}$.
	Let $\phihat\in\Phi(L, W, B, C, D)$ be the DP-DeepSets defined in \eqref{eq:definition-phihat} with
    \begin{align*}
        L \simeq \log \frac{d \cdot \epsilon N}{\log\delta^{-1}}, 
        \quad
        W \simeq d, \quad
        B \simeq \frac{d \cdot \epsilon N}{\log\delta^{-1}}, \quad
        D = k, \quad 
        C \simeq k \vee  \log \qty(\frac{\epsilon N}{\log \delta^{-1}}).
    \end{align*}
    Moreover, suppose that we set the DP noise scale as 
    $\sdp^2 \simeq \frac{C^2 \log(\delta^{-1})}{\epsilon^2 N^2}$.
    Then, we have
    \begin{align*}
        &\bE\qty[\norm{\phihat(p_{1:N}, \xi) - w}^2] 
        \lesssim 
        \qty[
            \frac{d^2}{\sqrt{M}}
            + \frac{d}{N}
            + \frac{k^3}{\epsilon^2 N^2} 
            + \int_k^d \frac{\dd{t}}{t^{2\gamma}}
        ] \poly\log (N, \epsilon^{-1}, \delta^{-1}),
    \end{align*}
    where the expectation is taken over $\pidx{1}_{1:N}, \ldots, \pidx{M}_{1:N}$, $\xiidx{1}, \ldots, \xiidx{M}$, 
    $p_{1:N}$, $\xi$ and $w$.
\end{theorem}
The first term, $d^2/\sqrt{M}$, corresponds to the generalization gap in hypernetwork learning: learning a map from datasets to parameters using $M$ datasets incurs an error of order $O(1/\sqrt{M})$.
The second term, $d/N$, arises from estimating a $d$-dimensional variable from $N$ samples.

The third term comes from the DP noise, while the final term represents the approximation error due to the low-dimensional representation of the dataset. 
Here, a trade-off in $k$ appears: increasing $k$ reduces the error caused by embedding the data into a low-dimensional space, 
but it also increases the dimension of the DP noise and hence the error induced by that noise.
The upper bound can be minimized by balancing these two terms, namely by choosing
$k = \Theta\qty(\min\qty{d, (\epsilon N)^{\frac{1}{\gamma+1}}})$.
With this choice and sufficiently large $M$, the resulting error bound becomes
$\Otilde (\frac{d}{N} + \qty(\frac{1}{\epsilon N})^{\frac{2\gamma-1}{\gamma+1}})$.

\subsection{Lower Bound for DP-SGD}

Next, we present the lower bound on the error in estimating the parameter $w$ using DP-SGD for the target model $\psi_\omega(x) = \omega^\top x$.
We run DP-SGD on the following empirical loss for linear regression:
\begin{align*}
    \scL(\omega) := \frac1{N} \sum_{n=1}^N \scL_n(\omega), 
    \quad\text{where}\quad
    \scL_n(\omega) := \frac12 (\omega^\top x_n - y_n)^2.
\end{align*}
For simplicity, we consider DP-SGD using the full batch at each step.
In other words, we analyze the parameter obtained by the following update: 
\begin{align*}
    \omega_t = \omega_{t-1} - 
    \frac{\eta}{N}\qty{
        \sum_{n=1}^N \clip_C(\nabla\scL_n(\omega_{t-1}))
        + \xi_t
    },
    \quad \xi_t \sim \scN(0, s^2 I_d),
\end{align*}
where $\eta > 0 $ is a learning rate, and $\xi_t$ is DP noise added to the sum of clipped gradients, as in \cref{sec:dp-sgd-background}.
After $T$ steps of this update, the resulting parameter satisfies the bound stated in the following theorem.
The proof is provided in \cref{app:proof-of-lower-bound-dpsgd}.
\begin{theorem}\label{thm:lower-bound-dpsgd}
    Let $T \in \bZ_{>0}$ be the number of steps of DP-SGD.
    Suppose that the initial point is at the origin, i.e., $\omega_0 = 0$.
    Moreover, assume that it holds $\eta\in(0, 1)$, $N \geq N_0 \log N_0$ and $C \gtrsim d$ with $N_0 \simeq d \sqrt{(T \log T) \epsilon^{-2} \log(\delta^{-1})}$.
    Additionally, we set the DP noise scale as $s^2 \simeq \frac{C^2 T \log(\delta^{-1})}{\epsilon^2}$.
    Then, it holds
    \begin{align*}
        \bE\qty[\norm{\omega_T - w}^2]
        \gtrsim \frac{d}{N} 
        + \min\qty{\frac{d^3 \log(\delta^{-1})}{\epsilon^2 N^2}, 1}, 
    \end{align*}
    where the expectation is taken over $p_{1:N}$, $\xi_1, \ldots, \xi_T$, and $w$.
\end{theorem}
The first term, $d/N$, is the error of estimating a $d$-dimensional parameter from $N$ data points, as in \cref{thm:upper-bound-hyper-network}.
The second term is due to DP noise.
Unlike DP-DeepSets, DP-SGD necessarily adds $d$-dimensional noise repeatedly, yielding an error term that increases with $d$.

\paragraph{Comparison in the high-dimensional regime}
Here, we compare the errors of DP-DeepSets $\Eours$ and DP-SGD $\Esgd$ in the high-dimensional regime where $d\to\infty$ and $N\to\infty$ with $N/d^{1+\zeta} = \Theta(1)$ for some $\zeta > 0$.
First, DP-DeepSets matches the DP-SGD lower-bound scaling in the worst case.
Indeed, if we choose $k = d$ and take $M$ sufficiently large in \cref{thm:upper-bound-hyper-network}, then we obtain
$
\Eours = O\left(\frac{d}{N} + \frac{d^3}{\epsilon^2 N^2}\right),
$
which matches the lower-bound rate for $\Esgd$.
Moreover, in the regime $\zeta\in(0, 1/2)$, we have
\begin{align*}
    \Eours 
    = \Otilde \qty(
        \frac{1}{d^\zeta} + \qty(\frac{1}{\epsilon d^{1+\zeta}})^{\frac{2\gamma-1}{\gamma+1}}
    )  = o_d(1), 
     \quad
    \Esgd = \Omegatilde \qty(
        \frac{1}{d^\zeta} + \min\qty{\frac{d^{1-2\zeta}}{\epsilon^2}, 1}
    ) = \Omega(1),
\end{align*}
if we choose
$
k = \Theta\left((\epsilon N)^{\frac{1}{\gamma+1}}\right)
$
for DP-DeepSets.
These observations show that sufficiently trained DP-DeepSets can match DP-SGD in general and, in a nontrivial high-dimensional regime, can substantially outperform it, demonstrating the effectiveness of our proposed method.

\section{Experiments: LoRA Finetuning of Diffusion Models}\label{sec:experiments}

In this section, we demonstrate the effectiveness of our approach in LoRA finetuning of diffusion models.

\paragraph{Models and Datasets}
As the pretrained model, we use the checkpoint for the unconditional diffusion model trained on ImageNet64~(\texttt{imagenet64\_uncond\_100M\_1500K.pt}), provided in the codebase\footnote{\url{https://github.com/openai/improved-diffusion}} of \citet{nichol2021improved}.
We consider \emph{finetuning on small private datasets} constructed from CIFAR-10~\citep{krizhevsky2009learning}.
Specifically, we first divide the 60,000 data points (50,000 training and 10,000 test images) into two subsets, treating one subset of 30,000 data points as public data and the remaining half as private data.
Next, we embed the data into 2,048-dimensional representations using Inception~\citep{szegedy2016rethinking}.
We then select 100 representative data points from the private chunk and construct a \emph{set of 100 private datasets, each containing 128 data points}, by taking the $k$-nearest neighbors of each representative point with $k=128$.
Here, the representative points are selected to minimize overlap between the resulting datasets.

\paragraph{Settings for DP-DeepSets}
To feed image datasets into DP-DeepSets, we use CLIP~\citep{radford2021learning} to embed images into 512-dimensional representations.
The base diffusion model is kept fixed, 
and DP-DeepSets generates only the LoRA parameters.
The LoRA parameters are partitioned into 318 tokens, so that each token contains only components of the LoRA parameters corresponding to the same location.

When training DP-DeepSets, we employ the public chunk of CIFAR-10 containing 30,000 data points.
More concretely, we take the $k$-nearest neighbors of each data point in the Inception embedding space with $k=128$ to obtain 30,000 public datasets with 128 data points.
Then, we fit the parameters $\theta$ of DP-DeepSets using the loss function $\scL_\HN$ defined in \eqref{eq:hypernet-loss} with $N=128$ and $M=30{,}000$, where $\pidx{m}_{1:N}$ denotes the CLIP embeddings of the $k$-NN set for the $m$-th data point in the public chunk.

\paragraph{Comparison baselines}
As the most basic baseline, we apply DP-SGD to each private dataset separately.
In addition, we consider two types of public-data-guided baselines.
First, we consider finetuning the model on public data and then performing DP finetuning on private data. Specifically, we first finetune the model with standard Adam on the public chunk consisting of 30,000 examples, and then fit the model to the private data using DP-SGD.
Second, we compare our method with PDA-DPMD~\citep{amid2022public}. 
This method performs mirror descent using the loss on public data as the mirror map. In our experiments, we use the noise prediction loss for denoising on public data as the mirror map.

\paragraph{Results}
For each method, we conduct finetuning on 100 datasets, 
which yields 100 LoRA finetuned models.
Then, we generate 1,024 images from each finetuned model, and compute the FID between the generated images and the private dataset corresponding to the model.
For the DP cost, we fix $\delta=10^{-5}$, and conduct experiments with various values of $\epsilon$.
We report the average and the standard deviation of FID scores over 100 datasets.

The results are shown in \cref{fig:fid-for-epsilons}. 
We also display the generated samples from the model finetuned by our method with $\epsilon=2$ in \cref{fig:generated-images}.
We describe some observations below:

\begin{itemize}\setlength{\leftskip}{-8mm}\setlength{\itemsep}{1mm}
    \item \textbf{Superiority of DP-DeepSets over DP-SGD:}
        We observe that the FID scores of our method are much lower than those of DP-SGD.
        In particular, as $\epsilon$ decreases, DP-SGD exhibits a substantial increase in FID for $\epsilon \leq 16$, whereas the proposed method maintains nearly the same performance down to around $\epsilon = 1$.
        This is because, unlike DP-SGD, which repeatedly adds parameter-dimensional noise, the proposed method adds DP noise to a low-dimensional dataset representation, and does so only once.
        This reduces the scale of the DP noise and mitigates its adverse effects.
        This effect is particularly pronounced in our setting, where the private dataset is small, containing only 128 samples.
    \item \textbf{Stronger results than other public-data-guided approaches:}
        We observe that the proposed method outperforms baselines that use public data.
        This is mainly because the baseline methods rely on gradient-based parameter updates, and thus cannot avoid adding noise multiple times to high-dimensional parameters.
        This performance gap is also attributable to how public data are exploited.
        In the two comparison baselines, gradients computed from public and private data are handled separately and are simply used independently for parameter updates\footnote{
            In the exact update of PDA-DPMD, the auxiliary loss augmented with the public loss is minimized at each step, 
            so the private and public data interact nonlinearly.
            However, because this procedure is inefficient, the original paper \citep{amid2022public} proposes an approach that linearly approximates the public loss using its gradient.
            We adopt this approach in our experiments as well.
        }.
        In contrast, in our approach, the hypernetwork learns ``how to learn'' from public datasets.
        Given private data, it can use the network learned from public data to determine how to generate parameters, enabling more flexible parameter generation than simply using the public and private gradients separately for parameter updates.
\end{itemize}

\section{Conclusion}\label{sec:conclusion}

In this paper, we propose a new framework for DP learning using a public hypernetwork.
Specifically, we develop DP-DeepSets, a hypernetwork architecture for differentially private learning.
Since it performs DP-noise perturbation on a low-dimensional representation of private datasets only once, 
it suppresses the scale of DP noise, thereby achieving higher utility than gradient-based methods.
We also propose an end-to-end pipeline to train DP-DeepSets using public datasets.
In a theoretical analysis in a synthetic setting, we show that the proposed method performs better because it learns how to encode low-dimensional structure into dataset embeddings during training with public datasets.
Moreover, we conduct experiments on LoRA finetuning of diffusion models and demonstrate that our approach achieves higher utility than DP-SGD and other public-data-guided methods.

\vspace{-2mm}
\paragraph{Limitations and Future Work} 
While our proposed framework is broadly applicable to various modalities and tasks, in this paper, we focused on LoRA finetuning of diffusion models, particularly in the context of data generation. 
Extending our experiments to other modalities, including language, as well as to tasks beyond data generation, remains an important direction for future work.
In addition, our experiments considered a small-data regime, where gradient-based DP learning is often severely hindered by the limited dataset size.
A promising future direction is to investigate how the effectiveness of our method scales with dataset size, and to clarify up to what regime it can serve as a practical alternative to gradient-based approaches.


\bibliography{ref}

\newpage
\appendix

\section{Proof of \cref{cor:gaussian-mechanism-for-mean}}\label[appendix]{app:proof-hypernet-dp}

Here, we prove \cref{cor:gaussian-mechanism-for-mean}.
\begin{proof}
    Let $p_{1:N}$ and $p'_{1:N}$ be adjacent datasets.
    Without loss of generality, we assume that $p_n = p'_n$ for $n=2, \ldots, N$.
    Then, we have
	\begin{equation*}
        \| E(p_{1:N}, \xi) - E(p'_{1:N}, \xi)\|
        \leq \frac1N \norm{\clip_C(\phi_1(p_1)) - \clip_C(\phi_1(p'_1))}
        \leq \frac{2C}{N}.
	\end{equation*}
    \cref{prop:DP-gaussian-mechanism} implies that the map $E$ is $(\epsilon, \delta)$-DP.
    Since the network $\phi_2$ does not depend on $p_{1:N}$, the post-processing property of DP yields the claim.
\end{proof}

\section{Proof of \cref{thm:upper-bound-hyper-network}}\label[appendix]{app:proof-of-upper-bound-hn}

We define the empirical risk as
\begin{align*}
    \Rhat(\phi) &:= \frac{1}{M} \sum_{m=1}^M \qty[
        \frac1N \sum_{n=1}^N \qty{
            \phi(\pidx{m}_{1:N}, \xiidx{m})^\top \xidx{m}_n - \yidx{m}_n
        }^2
    ].
\end{align*}

\subsection{Upper Bound of the Empirical Error}

We first prove the following lemma, 
which upper bounds the empirical parameter estimation error
by the expected error.
\begin{lemma}\label{lem:upper-bound-empirical-error}
    Let $\phihat$ be the empirical minimizer in $\Phi(L, W, B, C, D)$, i.e.,
    \begin{align*}
        \phihat := \argmin_{\phi\in\Phi(L, W, B, C, D)} \Rhat(\phi).
    \end{align*}
    Moreover, let $\phiast \in \Phi(L, W, B, C, D)$, and define
    $\Delta_N$ as
    \begin{align*}
        \Delta_N := \bE \qty[ \norm{
            \phiast(p_{1:N}, \xi) 
            - w
        }^2].
    \end{align*}
    Then, for some constant $\cEmpirExpect$, it holds
    \begin{align*}
        \bE \qty[\frac{1}{M} \sum_{m=1}^M \norm{
            \phihat(\pidx{m}_{1:N}, \xiidx{m}) 
            - \phiast(\pidx{m}_{1:N}, \xiidx{m})
        }^2] \leq 
        \cEmpirExpect R_N \log^2\qty(\rme + \frac{4F^2}{R_N}),
    \end{align*}
    where
    \begin{align*}
        R_N := \qty(1 + \sqrt{\frac{\log(2N)}{d}})^4
        \cdot \qty(\Delta_N + \frac{d}{N} + \frac{d^2}{\sqrt{M}}).
    \end{align*}
\end{lemma}
\begin{proof}
First, it holds $\Rhat(\phihat) \leq \Rhat(\phiast)$, 
which implies
\begin{align*}
    &\frac{1}{M} \sum_{m=1}^M \qty[
        \frac1N \sum_{n=1}^N \qty{
            \phihat(\pidx{m}_{1:N}, \xiidx{m})^\top \xidx{m}_n - \yidx{m}_n
        }^2
    ] \\
    &\hspace{150pt}\leq 
    \frac{1}{M} \sum_{m=1}^M \qty[
        \frac1N \sum_{n=1}^N \qty{
            \phiast(\pidx{m}_{1:N}, \xiidx{m})^\top \xidx{m}_n - \yidx{m}_n
        }^2
    ]
\end{align*}
Let $\hidx{m} := \phihat(\pidx{m}_{1:N}, \xiidx{m}) - \phiast(\pidx{m}_{1:N}, \xiidx{m})$
and $\Deltaidx{m} := \widx{m} - \phiast(\pidx{m}_{1:N}, \xiidx{m})$.
Since it holds
\begin{align*}
\yidx{m}_n 
= (\widx{m})^\top \xidx{m}_n + \zidx{m}_n
= (\phiast(\pidx{m}_{1:N}, \xiidx{m}) + \Deltaidx{m})^\top \xidx{m}_n + \zidx{m}_n, 
\end{align*}
we have
\begin{align*}
    &\frac{1}{M} \sum_{m=1}^M \qty[
        \frac1N \sum_{n=1}^N \qty{
            \qty(\hidx{m})^\top \xidx{m}_n 
            - \qty(\Deltaidx{m})^\top \xidx{m}_n 
            - \zidx{m}_n
        }^2
    ] \\
    &\hspace{150pt} \leq 
    \frac{1}{M} \sum_{m=1}^M \qty[
        \frac1N \sum_{n=1}^N \qty[
            (\Deltaidx{m})^\top \xidx{m}_n + \zidx{m}_n
        ]^2
    ].
\end{align*}
We have
\begin{align*}
    \frac{1}{M} \sum_{m=1}^M \qty[
        (\hidx{m})^\top \Sigmaidx{m} \hidx{m}
        - 2 (\hidx{m})^\top \Sigmaidx{m} \Deltaidx{m}
        - 2 (\gidx{m})^\top \hidx{m}
    ] \leq 0, 
\end{align*}
where $\Sigmaidx{m} := \frac1N \sum_{n=1}^N \xidx{m}_n (\xidx{m}_n)^\top$ 
and $\gidx{m} := \frac1N \sum_{n=1}^N \zidx{m}_n \xidx{m}_n$.
This implies that
\begin{align*}
    \frac{1}{M} \sum_{m=1}^M \qty{
        (\hidx{m})^\top \Sigmaidx{m} \hidx{m}
    }
    \leq \frac{2}{M} \sum_{m=1}^M \qty{
        (\hidx{m})^\top \qty(\Sigmaidx{m} \Deltaidx{m} + \gidx{m})
    }.
\end{align*}
We have
\begin{align*}
    \sum_{m=1}^M (\hidx{m})^\top \Sigmaidx{m} \hidx{m}
    &\geq  \sum_{m=1}^M\lambdaidx{m}_{\min} \norm{\hidx{m}}^2
    \geq  \lambda_{\min} \sum_{m=1}^M \norm{\hidx{m}}^2,
\end{align*}
where $\lambda_{\min} := \min_{m=1, \ldots, M} \lambdaidx{m}_{\min}$,
and
\begin{align*}
    \sum_{m=1}^M \qty{
        (\hidx{m})^\top \qty(\Sigmaidx{m} \Deltaidx{m} + \gidx{m})
    }
    &\leq \sum_{m=1}^M \norm{\hidx{m}} \norm{\Sigmaidx{m} \Deltaidx{m} + \gidx{m}} \\
    &\leq \sqrt{\sum_{m=1}^M \norm{\hidx{m}}^2}
        \sqrt{\sum_{m=1}^M \norm{\Sigmaidx{m} \Deltaidx{m} + \gidx{m}}^2},
\end{align*}
Therefore, we have
\begin{align*}
    \frac{\lambda_{\min}^2}{M} \sum_{m=1}^M \norm{\hidx{m}}^2
    \leq \frac{4}{M}\sum_{m=1}^M \norm{\Sigmaidx{m} \Deltaidx{m} + \gidx{m}}^2.
\end{align*}
To bound the right-hand side, we bound its expectation as follows:
\begin{align*}
    \bE \qty[\norm{\Sigmaidx{m} \Deltaidx{m} + \gidx{m}}^2]
    &\leq 2\bE \qty[\norm{\Sigmaidx{1} \Deltaidx{1}}^2] 
        + 2\bE \qty[\norm{\gidx{1}}^2] \\
    &\leq 2\bE \qty[\norm{\Sigmaidx{1}}_{\rmop}^2 \norm{\Deltaidx{1}}^2] 
        + 2\bE \qty[\norm{\gidx{1}}^2] \\
    &\leq 4\bE \qty[\norm{\Deltaidx{1}}^2] 
        + 4\bE \qty[\norm{I - \Sigmaidx{1}}_{\rmop}^2 \norm{\Deltaidx{1}}^2] 
        + 2\bE \qty[\norm{\gidx{1}}^2] \\
    &\leq 4 \Delta_N
        + 4F^2 \bE \qty[\norm{I - \Sigmaidx{1}}_{\rmop}^2] 
        + 2\bE \qty[\norm{\gidx{1}}^2].
\end{align*}
Using \cref{lem:eigenvalue-concentration}, the second term can be upper bounded as
\begin{align*}
    \bE \qty[\norm{I - \Sigmaidx{1}}_{\rmop}^2] 
    \leq \frac{\cEigenExpect d}{N}, 
\end{align*}
for some constant $\cEigenExpect$.
Moreover, the third term can be upper-bounded as
\begin{align*}
    \bE \qty[\norm{\gidx{1}}^2]
    = \frac{1}{N^2} \sum_{n=1}^N \bE \qty[\|\zidx{m}_n \xidx{m}_n\|^2]
    = \frac{d}{N}.
\end{align*}
Therefore, we have
\begin{align*}
    \bE \qty[\norm{\Sigmaidx{m} \Deltaidx{m} + \gidx{m}}^2]
    \leq 4 \Delta_N + \frac{(1 + \cEigenExpect) d}{N}.
\end{align*}
We apply Bernstein's inequality to the i.i.d. random variables
\begin{align*}
    Z_m := \norm{\Sigmaidx{m}\Deltaidx{m}+\gidx{m}}^2,
    \quad m=1,\ldots,M.
\end{align*}
The outputs of the networks are clipped, and $w$ is supported on a bounded set;
hence there exists a constant $C_0>0$ such that $\norm{\Deltaidx{m}}\leq C_0$ for all $m$.
Since $\norm{x_n}=\sqrt d$, we have $\norm{\Sigmaidx{m}}_{\rmop}\leq d$ deterministically, and therefore
\begin{align*}
    Z_m
    \leq 2 C_0^2 d^2 + 2\norm{\gidx{m}}^2.
\end{align*}
Conditionally on $x_{1:N}^{(m)}$, the vector $\gidx{m}=N^{-1}\sum_{n=1}^N z_n^{(m)}x_n^{(m)}$
is Gaussian with covariance $\Sigmaidx{m}/N$.
Thus $\norm{\gidx{m}}^2$ is sub-exponential, and the preceding display implies that
there exists a constant $C_1>0$ such that
\begin{align*}
    \norm{Z_m - \bE Z_m}_{\psi_1} \leq C_1 d^2,
\end{align*}
for all $m$.
Bernstein's inequality for sub-exponential random variables then implies that, with probability $1 - \eta/2$, it holds 
\begin{align*}
     \frac{1}{M}\sum_{m=1}^M \norm{\Sigmaidx{m} \Deltaidx{m} + \gidx{m}}^2
     &\leq \bE \qty[\norm{\Sigmaidx{m} \Deltaidx{m} + \gidx{m}}^2] 
        + \cBernst d^2 \qty(\sqrt{\frac{\log(2/\eta)}{M}} + \frac{\log(2/\eta)}{M}) \\
    &\leq 4 \Delta_N + \frac{(1 + \cEigenExpect) d}{N}
        + \cBernst d^2 \qty(\sqrt{\frac{\log(2/\eta)}{M}} + \frac{\log(2/\eta)}{M}).
\end{align*}
Hence, we have
\begin{align*}
    \frac{1}{M} \sum_{m=1}^M \norm{\hidx{m}}^2
    \leq \frac{4}{\lambda_{\min}^2}
    \cdot \qty{4 \Delta_N + \frac{(1 + \cEigenExpect) d}{N}
        + \cBernst d^2 \qty(\sqrt{\frac{\log(2/\eta)}{M}} + \frac{\log(2/\eta)}{M})
    },
\end{align*}
with probability $1 - \eta/2$.
Using \cref{lem:eigenvalue-concentration}, 
if $N \geq 2d$ and $d \geq 3$ (and therefore $N \geq d+3$), 
we have
\begin{align*}
    \bP \qty(\lambdamin(\Sigma) \leq \frac{\sqrt{\eta}}{2 K_{N,d}^2}) 
    \leq 2\qty(\frac{\eta^{1/4}}{\sqrt{2}})^{N-d+1}
    \leq 2 \cdot \qty(\frac{\eta^{1/4}}{\sqrt{2}})^4
    = \frac\eta2.
\end{align*}
This implies that it holds
\begin{align*}
    \frac{1}{\lambdamin(\Sigma)} 
    &\leq \frac{2 K_{N,d}^2}{\sqrt{\eta}}
    \leq \frac{2\cEigenMin^2}{\sqrt{\eta}} \frac{N}{d}
            \frac{(\sqrt{d} + \sqrt{2\log(2N)})^2}
                {(\sqrt{N}-\sqrt{d})^2} \\
    &\leq \frac{2\cEigenMin^2}{\sqrt{\eta}}
            \frac{\qty(1 + \sqrt{\frac{2\log(2N)}{d}})^2}
                {\qty(1 - \sqrt{\frac{d}{N}})^2}
    \leq \frac{2\qty(1-\frac{1}{\sqrt2})^{-2}\cEigenMin^2}{\sqrt{\eta}}
            \qty(1 + \sqrt{\frac{2\log(2N)}{d}})^2,
\end{align*}
with probability $1 - \eta/2$.
Therefore, we have
\begin{align*}
    \frac{1}{M} \sum_{m=1}^M \norm{\hidx{m}}^2
    \leq \frac{\cEmpirHp}{\eta}
            \qty(1 + \sqrt{\frac{\log(2N)}{d}})^4
    \cdot \qty{\Delta_N + \frac{d}{N}
        + d^2 \qty(\sqrt{\frac{\log(2/\eta)}{M}} + \frac{\log(2/\eta)}{M})
    },
\end{align*}
with probability $1 - \eta$.
Therefore, for any $\eta \in (0,1)$, we have
\begin{align*}
    \frac{1}{M} \sum_{m=1}^M \norm{\hidx{m}}^2
    \leq \frac{\cEmpirExpect}{\eta} R_N \log\qty(\frac{\rme}{\eta}),
\end{align*}
with probability at least $1-\eta$.
On the other hand, since the output of every network in $\Phi(L,W,B,C,D)$ is clipped by $\clip_F$, it holds deterministically that
\begin{align*}
    \frac{1}{M} \sum_{m=1}^M \norm{\hidx{m}}^2 \leq 4F^2.
\end{align*}
The preceding high-probability bound implies that
\begin{align*}
    \bP\qty(\frac{1}{M} \sum_{m=1}^M \norm{\hidx{m}}^2 > t)
    \leq \min\qty{1, \frac{\cEmpirExpect R_N}{t}
        \log\qty(\rme + \frac{t}{R_N})}
\end{align*}
for all $t>0$.
Using this tail bound together with the deterministic bound above, we obtain
\begin{align*}
    \bE\qty[\frac{1}{M} \sum_{m=1}^M \norm{\hidx{m}}^2]
    &= \int_0^{4F^2}
        \bP\qty(\frac{1}{M} \sum_{m=1}^M \norm{\hidx{m}}^2 > t) \dd{t} \\
    &\leq \cEmpirExpect R_N
        \log^2\qty(\rme + \frac{4F^2}{R_N}),
\end{align*}
by increasing $\cEmpirExpect$ if necessary.
This completes the proof.
\end{proof}

\subsection{Difference Between the Empirical Error and the Test Error}

Here, we prove the following lemma, 
which upper-bounds the difference between the empirical error and the test error between two MLPs.
Eventually, we will use this lemma by setting $\phi$ as the empirical minimizer, and $\phiast$ as the network constructed to achieve small error.
\begin{lemma}\label{lem:diff-empirical-error-test-error}
    For any $\phi, \phiast \in \Phi(L, W, B, C, D)$, it holds
    \begin{align*}
        &\bE \qty[
            \sup_{\phi\in\Phi} \abs{\bE[f_\phi(p_{1:N}, \xi)] - \frac1M\sum_{m=1}^M f_\phi(\pidx{m}_{1:N}, \xiidx{m})} 
        ] \\
        &\qquad \leq \frac{\cGenErr LW}{\sqrt{M}}
            \sqrt{
            \log\qty{
                d L (B \vee 1) (W+1) (C \vee 1)
                \cdot \epsilon^{-1} \log(\delta^{-1})
            }}.
    \end{align*}
    where $\cGenErr$ is a constant.
\end{lemma}

\begin{proof}
For $\phi\in\Phi(L, W, B, C, D)$, let $\param(\phi) \in [-B, B]^{\scC}$ denote the vector obtained by concatenating its parameters into a single vector.
In the following, we consider the regime in which $W \geq \max\qty{d, d', D}$,
where $d$ and $d'$ denote the dimension of input and output, respectively.

For a fixed $\phiast \in \Phi(L, W, B, C, D)$,
we define the function class $\scF(L, W, B, C, D)$ as
\begin{align*}
    &\scF(L, W, B, C, D) \\
    &\quad := \left\{
        f_\phi: (p_{1:N}, \xi) \mapsto 
        \norm{\phi(p_{1:N}, \xi) - \phi^\ast(p_{1:N}, \xi)}^2
        \relmiddle| \phi\in\Phi(L, W, B, C, D)
    \right\}.
\end{align*}
We first remark that the number of parameters $\scD(L, W, D)$ of neural networks in $\Phi(L, W, B, C, D)$ is bounded as
\begin{align*}
    \scD(L, W, D)
    := d'W + W^2 \cdot (L-2) + WD + DW + W^2 \cdot (L-2) + Wd
    \leq 2LW^2.
\end{align*}
There exists a $\delta$-cover $\scC$ of $\{\param(\phi) \mid \phi \in \Phi(L, W, B, C, D)\} = [-B, B]^{\scD}$ with $|\scC| \leq (B/\delta)^\scD$
with respect to the $\|\cdot\|_\infty$ norm.
Therefore, for any $\phi \in \Phi(L, W, B, C, D)$, we can choose $\param(\phi') \in \scC$ such that
$\|\param(\phi) - \param(\phi')\|_\infty \leq \delta$.
Then, for any $x_{1:N} \in \bR^{d\times N}$, $y_{1:N} \in \bR^{N}$
and $\xi \in \bR^{D}$, we have
\begin{align*} 
    \|\phi_1(p_n) - \phi'_1(p_n)\|_\infty
    \leq L (B \vee 1)^{L-1} (W+1)^L \cdot \norm{p_n}_\infty \delta.
\end{align*}
The proof of this inequality can be found in \citet{suzuki2018adaptivity}, for example.
Then, we have
\begin{align*} 
    &\norm{\qty(\frac1N\sum_{n=1}^N \clip_C(\phi_1(p_n)) + \xi)
        - \qty(\frac1N\sum_{n=1}^N \clip_C(\phi'_1(p_n)) + \xi)}_\infty \\
    &\hspace{100pt}
        \leq \frac{1}{N} \sum_{n=1}^N 
        \|\clip_C(\phi_1(p_n)) - \clip_C(\phi'_1(p_n))\|_\infty \\
    &\hspace{100pt}
        \leq \frac{1}{N} \sum_{n=1}^N 
        \|\phi_1(p_n) - \phi'_1(p_n)\|_\infty \\
    &\hspace{100pt}
        \leq L (B \vee 1)^{L-1} (W+1)^L 
            \cdot \qty(\frac1N\sum_{n=1}^N \norm{p_n}_\infty) \delta.
\end{align*}
We can bound the term $\frac1N\sum_{n=1}^N \norm{p_n}_\infty$ as follows:
\begin{align*}
    \frac1N\sum_{n=1}^N \norm{p_n}_\infty
    &\leq \frac1N\sum_{n=1}^N \qty(\norm{x_n}_2 + \abs{y_n})
    \leq \frac1N\sum_{n=1}^N \qty(\norm{x_n}_2 
            + \norm{x_n}_2 \norm{w}_2 + \abs{z_n}) \\
    &\leq (1 + \|\alpha\|) \frac1N\sum_{n=1}^N \norm{x_n}_2 
            + \frac1N\sum_{n=1}^N \abs{z_n} \\
    &\leq d(1 + \|\alpha\|)+ \frac1N\sum_{n=1}^N \abs{z_n} \\
    &\leq d \cdot \frakC_1(z_{1:N}), 
\end{align*}
where we define $\frakC_1(z_{1:N}) := 1 + \|\alpha\| + \frac1N\sum_{n=1}^N \abs{z_n}$.
Therefore, we have
\begin{align*} 
    &\norm{\qty(\frac1N\sum_{n=1}^N \clip_C(\phi_1(p_n)) + \xi)
        - \qty(\frac1N\sum_{n=1}^N \clip_C(\phi'_1(p_n)) + \xi)}_\infty \\
    &\hspace{100pt}
        \leq d L (B \vee 1)^{L-1} (W+1)^L 
            \cdot \frakC_1(z_{1:N}) \delta.
\end{align*}
Since $\clip_F$ is 1-Lipschitz and $\phi_2$ is $(BW)^L$-Lipschitz, 
we have
\begin{align*}
    &\|\phi(p_{1:N},\xi) - \phi'(p_{1:N}, \xi)\|_\infty \\
    &\quad\leq \norm{
        \phi_2\qty(\frac1N\sum_{n=1}^N \clip_C(\phi_1(p_n)) + \xi)
        - \phi'_2\qty(\frac1N\sum_{n=1}^N \clip_C(\phi'_1(p_n)) + \xi)
        }_\infty \\
    &\quad\leq \norm{
        \phi_2\qty(\frac1N\sum_{n=1}^N \clip_C(\phi_1(p_n)) + \xi)
        - \phi_2\qty(\frac1N\sum_{n=1}^N \clip_C(\phi'_1(p_n)) + \xi)
        }_\infty \\
    &\hspace{50pt}
        + \norm{
        \phi_2\qty(\frac1N\sum_{n=1}^N \clip_C(\phi'_1(p_n)) + \xi)
        - \phi'_2\qty(\frac1N\sum_{n=1}^N \clip_C(\phi'_1(p_n)) + \xi)
        }_\infty \\
    &\quad\leq (BW)^L \cdot d L (B \vee 1)^{L-1} (W+1)^L 
            \cdot \frakC_1(z_{1:N}) \delta
        +  L (B \vee 1)^{L-1} (W+1)^L \cdot \qty(C + \|\xi\|_\infty)
            \cdot \delta \\
    &\quad\leq (\frakC_1(z_{1:N}) + C + \norm{\xi}_\infty) 
        \cdot d L (B \vee 1)^{2L-1} (W+1)^{2L} \cdot \delta.
\end{align*}
By the definition of $f_\phi$, we have
\begin{align*}
    &\abs{f_\phi(p_{1:N}, \xi) - f_{\phi'}(p_{1:N}, \xi)} \\
    &\quad = \abs{
        \norm{\phi(p_{1:N}, \xi) - \phi^\ast(p_{1:N}, \xi)}^2
        - \norm{\phi'(p_{1:N}, \xi) - \phi^\ast(p_{1:N}, \xi)}^2
    } \\
    &\quad \leq 4F \cdot \abs{
        \norm{\phi(p_{1:N}, \xi) - \phi^\ast(p_{1:N}, \xi)}
        - \norm{\phi'(p_{1:N}, \xi) - \phi^\ast(p_{1:N}, \xi)}
    } \\
    &\quad \leq 4F
        \norm{\phi(p_{1:N}, \xi) - \phi'(p_{1:N}, \xi)} \\
    &\quad \leq 4F \sqrt{d}(\frakC_1(z_{1:N}) + C + \norm{\xi}_\infty)
        \cdot d L (B \vee 1)^{2L-1} (W+1)^{2L} \cdot \delta \\
    &\quad \leq 4F 
        \qty(\frac{1}{N} \sum_{n=1}^N \abs{z_n}
        + \norm{\xi}_\infty + \|\alpha\| + 2) 
        \cdot d^{3/2} (C \vee 1) L (B \vee 1)^{2L-1} (W+1)^{2L} \cdot \delta \\
    &\quad \leq 4F 
        \frakC_2(z_{1:N}, \xi)
        \cdot d^{3/2} (C \vee 1) L (B \vee 1)^{2L-1} (W+1)^{2L} \cdot \delta,
\end{align*}
where we define 
$\frakC_2(z_{1:N}, \xi) 
:= 4F \cdot (\frakC_1(z_{1:N}) + \norm{\xi}_\infty + 1)$.
Let $P_M$ be the empirical measure of $(p_{1:N}, w, \xi)$ defined by $M$ data.
Then, we have
\begin{align*}
    &\norm{f_\phi - f_{\phi'}}_{L^2(P_M)}^2 \\
    &= \int\abs{f_\phi(p_{1:N}, \xi) - f_{\phi'}(p_{1:N}, \xi)}^2 
        \dd P_M(p_{1:N}, \xi) \\
    &\leq \qty(
            4F \cdot d^{3/2} (C \vee 1) L (B \vee 1)^{2L-1} (W+1)^{2L} \cdot \delta
        )^2 \cdot \frac1M \sum_{m=1}^M \frakC_2(\zidx{m}_{1:N}, \xiidx{m})^2.
\end{align*}
Since it holds
\begin{align*}
    \bE \qty [\frac1M \sum_{m=1}^M \frakC_2(\zidx{m}_{1:N}, \xiidx{m})^2]
    = \bE_{z_{1:N}, \xi} \qty [\frakC_2(z_{1:N}, \xi)^2] 
    \leq \cMarkovFrakC^2 (\sdp^2 + 1) \log(2d), 
\end{align*}
for some constant $\cMarkovFrakC$, we have
\begin{align*}
    \frac1M \sum_{m=1}^M \frakC_2(\zidx{m}_{1:N}, \xiidx{m})^2
    \leq \eta^{-1} \bE \qty[
        \frac1M \sum_{m=1}^M \frakC_2(\zidx{m}_{1:N}, \xiidx{m})^2
    ]
    \leq \cMarkovFrakC^2 (\sdp^2 + 1) \log(2d) \eta^{-1},
\end{align*}
with probability $1 - \eta$.
Therefore, we have
\begin{align*}
    \norm{f_\phi - f_{\phi'}}_{L^2(P_M)}
    &\leq 4F \cMarkovFrakC 
        \cdot d^{3/2} (C \vee 1) L (B \vee 1)^{2L-1} (W+1)^{2L}
        \cdot \sqrt{(\sdp^2 + 1) \log(2d)} \eta^{-1/2} \delta.
\end{align*}
This implies that the set $\{f_\phi \mid \phi\in\Phi, \param(\phi)\in\scC\}$
is a $(4F \cMarkovFrakC 
        \cdot d^{3/2} (C \vee 1) L (B \vee 1)^{2L-1} (W+1)^{2L}
        \cdot \sqrt{(\sdp^2 + 1) \log(2d)} \eta^{-1/2} \delta)$-cover
of the function class $\Phi(L, W, B, C, D)$ with respect to $\|\cdot\|_{L^2(P_M)}$.
By replacing $\delta$ with $(4F \cMarkovFrakC 
        \cdot d^{3/2} (C \vee 1) L (B \vee 1)^{2L-1} (W+1)^{2L}
        \cdot \sqrt{(\sdp^2 + 1) \log(2d)} \eta^{-1/2})^{-1} \cdot \delta$, 
we obtain the following bound of the covering number of $\scF$:
\begin{align*}
    &\scN(\delta; \scF, \|\cdot\|_\infty) \\
    &\leq B^\scD \cdot 
    \qty{(4F \cMarkovFrakC 
        \cdot d^{3/2} (C \vee 1) L (B \vee 1)^{2L-1} (W+1)^{2L}
        \cdot \sqrt{(\sdp^2 + 1) \log(2d)} \eta^{-1/2})^{-1} 
        \cdot \delta}^{-\scD} \\
    &\leq \qty((B \vee 1)^{2L} (W+1)^{2L})^\scD 
        \cdot (
            4F \cMarkovFrakC 
            \cdot d^{3/2} (C \vee 1) L
            \cdot \sqrt{(\sdp^2 + 1) \log(2d)} \eta^{-1/2}
            \delta^{-1}
        )^{\scD} \\
    &\leq \qty{
        \frakC_3(L, W, B, C, d, \sdp) 
        \cdot \eta^{-1/2} \delta^{-1}
    }^{2L \scD},
\end{align*}
where 
$\frakC_3(L, W, B, C, d, \sdp) 
:= 4F \cMarkovFrakC 
    \cdot d^{3/2} L (B \vee 1) (W+1) (C \vee 1)
    \cdot (\sdp^2 + 1) \log(2d) $.
Therefore, we have
\begin{align*}
    \log \scN(\delta; \scF, \|\cdot\|_{L^2(P_M)})
    &\leq 2L \scD \log\qty{
        \frakC_3(L, W, B, C, d, \sdp) \eta^{-1/2} \delta^{-1}
    } \\
    &\leq 4 L^2 W^2 \log\qty{
        \frakC_3(L, W, B, C, d, \sdp) \eta^{-1/2} \delta^{-1}
    }.
\end{align*}
Using Dudley's entropy integral bound~\citep[Theorem~5.22]{wainwright2019high}, we have
\begin{align*}
    \frakRhat_M(\scF)
    \leq \frac{\cDudley}{\sqrt{M}}
        \int_0^\infty \sqrt{\log\scN(\delta; \scF, L_2(P_M))} \dd{\delta}.
\end{align*}
Since it holds
\begin{align*}
    \scN(\delta; \scF, \|\cdot\|_{L^2(P_M)}) \leq
    \scN(\delta; \scF, \|\cdot\|_\infty) = 1
\end{align*}
for $\delta \geq 4F^2$, we have
\begin{align*}
    \frakRhat_M(\scF) 
    \leq \frac{2\cDudley LW}{\sqrt{M}}
        \int_0^{4F^2} \sqrt{
            \log\qty{
                \frakC_3(L, W, B, C, d, \sdp) \eta^{-1/2} \delta^{-1}
            }
        } \dd{\delta}.
\end{align*}
Since it holds 
$\int_0^R \sqrt{\log(S/\delta)} \dd{\delta} \leq R \sqrt{\log(\rme S/R)}$ for all $S>0$ and $R>0$, we have
\begin{align*}
    \frakRhat_M(\scF)
    \leq \frac{8F^2\cDudley LW}{\sqrt{M}}
    \sqrt{
    \log\qty{
        \frakC_3(L, W, B, C, d, \sdp) \eta^{-1/2}
    }}.
\end{align*}
Therefore, we have
\begin{align*}
    \frakR(\scF) 
    & = \bE[\frakRhat_M(\scF)]
    \leq \int_0^1 \frac{2F\cDudley LW}{\sqrt{M}}
    \sqrt{
    \log\qty{
        \frakC_3(L, W, B, C, d, \sdp) \eta^{-1/2}
    }} \dd{\eta} \\
    &\leq \frac{\cGenErr LW}{\sqrt{M}}
        \sqrt{
        \log\qty{
            \frakC_3(L, W, B, C, d, \sdp)
        }},
\end{align*}
for some constant $\cGenErr$.
Symmetrization inequality implies that
\begin{align*}
    &\bE \qty[
        \sup_{\phi\in\Phi} \abs{\bE[f_\phi(p_{1:N}, \xi)] - \frac1M\sum_{m=1}^M f_\phi(\pidx{m}_{1:N}, \xiidx{m})} 
    ]\\
    &\leq 2\frakR(\scF) 
    \leq \frac{2\cGenErr LW}{\sqrt{M}}
        \sqrt{
        \log\qty{
            \frakC_3(L, W, B, C, d, \sdp)
        }} \\
    &\leq \frac{\cGenErr' LW}{\sqrt{M}}
        \sqrt{
        \log\qty{
            d L (B \vee 1) (W+1) (C \vee 1)
            \cdot (\sdp^2 + 1)
        }} \\
    &\leq \frac{\cGenErr'' LW}{\sqrt{M}}
        \sqrt{
        \log\qty{
            d L (B \vee 1) (W+1) (C \vee 1)
            \cdot \epsilon^{-1} \log(\delta^{-1})
        }}.
\end{align*}
where $\cGenErr'$ and $\cGenErr''$ are constants.
This completes the proof.
\end{proof}

\subsection{Constructing a Network with Small Error}

In this section, we prove the following theorem, which constructs the network that outputs a vector close to the true parameter given $N$ samples $p_{1:N}$ and DP noise $\xi$.
\begin{theorem}\label{thm:dp-ds-construct}
    For any $k = 1, \ldots, d$, 
    there exists a network $\phi \in \Phi(L, W, B, C, D)$ with
    \begin{align*}
        &L \leq \cConstruct \left\lceil 
            \log \frac{d \cdot \epsilon N}{\log\delta^{-1}}
            \right\rceil , \quad
        W \leq 6 d, \quad
        B \leq 
        \frac{\cConstruct d \cdot \epsilon N}{\log\delta^{-1}}, \\
        & D = k, \quad 
        C = \cConstruct \qty{ k \vee 
            \log \qty(\frac{\epsilon N}{\log \delta^{-1}})},
    \end{align*}
    satisfying
    \begin{align*}
        &\bE_{x_{1:N}, z_{1:N}, w, \xi} \qty[\norm{\phi(p_{1:N}, \xi) - w}^2]\\
        &\qquad \leq 
        \cDpdsErr \qty{
            k \qty(\frac{\log \delta^{-1}}{\epsilon^2 N^2} 
                \qty{ k  \vee 
                \log \qty(\frac{\epsilon N}{\log \delta^{-1}})}^2
            + \frac1N) + \int_{k}^d \frac{\dd{t}}{t^{2\gamma}}
        }, 
    \end{align*}
    where $\cDpdsErr$ and $\cConstruct$ are some constants.
\end{theorem}

To construct the network, we use the following lemma:
\begin{lemma}[\citet{schmidt2020nonparametric}]\label{lem:NN-approx-mult}
	Let $i \in \bZ_{\geq 0}$.
	There exists a neural network $\phi_{\mult} \in \Phi_\mlp(L, W, B)$
	with 
    \begin{align*}
        L = i+5, \quad W \leq 6, \quad B = 1,
    \end{align*}
    satisfying the following three conditions:
	\begin{itemize}
		\item $\abs{\phi_{\mult}(x, y) - xy} \leq 2^{-i}$ for all $(x, y) \in [0, 1]^2$.
		\item $\abs{\phi_{\mult}(x, y)} \leq 1$ for all $(x, y) \in \bR^2$.
		\item If $xy = 0$, it holds $\phi_{\mult}(x, y) = 0$.
	\end{itemize}
\end{lemma}
We prove the following corollary using the lemma above.
\begin{corollary}\label{cor:NN-approx-mult-extend}
	Let $\varepsilon > 0$, $\rho_x, \rho_y \geq 1$.
	There exists a neural network $\phi_{\mult} \in \Phi_\mlp(L, W, B)$
	with 
    \begin{align*}
        L \leq 6 + \left\lceil 
            \log_2\frac{4\rho_x\rho_y}{\varepsilon} 
        \right\rceil, \quad
        W \leq 6, \quad
        B \leq 4 \rho_x \rho_y, 
    \end{align*}
    satisfying the following three conditions:
	\begin{itemize}
		\item $\abs{\phi_{\mult}(x, y) - xy} \leq \varepsilon$ 
            for all $x \in [-\rho_x, \rho_x]$ and $y \in [-\rho_y, \rho_y]$.
        \item $\abs{\phi_{\mult}(x, y)} \leq \rho_x \rho_y$ 
            for all $x, y \in \bR$.
        \item $\abs{\phi_{\mult}(x, y) - xy} \leq \rho_x (\rho_y + \abs{y})$ 
            for all $x \in [-\rho_x, \rho_x]$ and $y \in \bR$.
	\end{itemize}
\end{corollary}
\begin{proof}
    Let $\phi_1, \phi_2, \phi_3 \in \Phi_\mlp(L, W, B)$ with
    \begin{align*}
        L = 1, \quad W \leq 2, \quad B = 1,
    \end{align*}
    be neural networks defined as
	\begin{align*}
	    \phi_1(x) := \frac{x}{2\rho_x} + \frac12, \quad 
        \phi_2(y) := \frac{y}{2\rho_y} + \frac12, \quad
        \phi_3(x, y) := \rho_x x + \rho_y y + \rho_x \rho_y.
	\end{align*}
    Then, we have
    \begin{align*}
        \phi_1(x) \phi_2(y)
        = \frac{xy}{4\rho_x \rho_y} 
            + \frac{x}{4\rho_x} + \frac{y}{4\rho_y} + \frac{1}{4}
        = \frac{xy}{4\rho_x \rho_y} + \frac{1}{4\rho_x \rho_y} \phi_3(x, y), 
    \end{align*}
    which implies
    \begin{align*}
        xy = 4 \rho_x \rho_y \phi_1(x) \phi_2(y) - \phi_3(x, y).
    \end{align*}
    
    Now, we have $\phi_1(x), \phi_2(y) \in [0, 1]$ 
    if $x \in [-\rho_x, \rho_x]$ and $y \in [-\rho_y, \rho_y]$.
    Therefore, \cref{lem:NN-approx-mult} implies that there exists a neural network $\psi \in \Phi_\mlp(L, W, B)$ with
    \begin{align*}
        L \leq 5 + \left\lceil \log_2\frac{4 \rho_x\rho_y}{\varepsilon} \right\rceil, \quad
        W \leq 6, \quad
        B \leq 1,
    \end{align*}
    satisfying the following two statements:
    \begin{itemize}
        \item If $x \in [-\rho_x, \rho_x]$ and $y \in [-\rho_y, \rho_y]$, 
            it holds
            $\abs{\psi(\phi_1(x), \phi_2(y)) - \phi_1(x) \phi_2(y)}
            \leq \frac{\varepsilon}{4 \rho_x \rho_y}$.
        \item For all $x, y \in \bR$, it holds
            $\abs{\psi(\phi_1(x), \phi_2(y))} \leq 1$.
    \end{itemize}
    
    We define 
    \begin{align*}
        \phi_4(x, y) &:= 4\rho_x\rho_y \psi(\phi_1(x), \phi_2(y)) - \phi_3(x, y), \\
        \phi_5(x, y) &:= \max\{-\rho_x\rho_y, \min\{\rho_x\rho_y, \phi_4(x, y) \}\}.
    \end{align*}
    Then we have $\phi_5 \in \Phi(L, W, B)$ with
    \begin{align*}
        L \leq 6 + \left\lceil 
            \log_2\frac{4\rho_x\rho_y}{\varepsilon} 
        \right\rceil, \quad
        W \leq 6, \quad
        B \leq 4 \rho_x \rho_y.
    \end{align*}
    Moreover, for all $x, y \in \bR$, we have
    \begin{align*}
        \abs{\phi_5(x, y)} \leq \rho_x \rho_y.
    \end{align*}
    
    For $x \in [-\rho_x, \rho_x]$ and $y \in [-\rho_y, \rho_y]$, 
    we have $\abs{xy} \leq \rho_x \rho_y$, which yields
    \begin{align*}
        \abs{\phi_5(x, y) - xy}
        &\leq \abs{\phi_4(x, y) - xy} \\
        &\leq 4\rho_x \rho_y 
            \abs{\psi(\phi_1(x), \phi_2(y)) - \phi_1(x) \phi_2(y)} 
        \leq \varepsilon.
    \end{align*}
    Moreover, for any $x, y \in \bR$, we have
    \begin{align*}
        \abs{\phi_5(x, y) - xy}
        \leq \abs{\phi_5(x, y)} + \abs{xy}
        \leq \rho_x \rho_y + \rho_x\abs{y}
        \leq \rho_x (\rho_y + \abs{y}).
    \end{align*}
    This completes the proof.
\end{proof}

Using the corollary above, we prove \cref{thm:dp-ds-construct}.
\begin{proof}[Proof of \cref{thm:dp-ds-construct}]
    We first define the matrix $U_k \in \bR^{d \times k}$ as
    \begin{align*}
        (U_k)_{i,j} = \begin{cases}
            1 & \text{if $i=j$}, \\
            0 & \text{otherwise}.
        \end{cases} 
    \end{align*}
    Then, both $U_k U_k^\top$ and $I - U_k U_k^\top$ are projection matrices.
    Therefore, we have $(I - U_k U_k^\top)^2 = I - U_k U_k^\top$.
    Moreover, since it holds
    \begin{align*}
        \bE[w w^\top] = \frac13 \diag(\alpha_1^2, \ldots, \alpha_d^2),
    \end{align*}
    we have
    \begin{align*}
        U_k^\top \bE[w w^\top] U_k
         = \frac13 \diag(\alpha_1^2, \ldots, \alpha_k^2).
    \end{align*}
	Hence, we have
	\begin{align*}
		\bE[\|U_k U_k^\top w - w\|^2]
        &= \bE[\|(I - U_k U_k^\top) w\|^2] \\
		&= \bE[w^\top (I - U_k U_k^\top)^2 w] \\
		&= \bE[w^\top (I - U_k U_k^\top) w] \\
		&= \tr((I - U_k U_k^\top) \bE[w w^\top]) \\
		&= \tr(\bE[w w^\top]) - \tr(U_k U_k^\top \bE[w w^\top]) \\
		&= \tr(\bE[w w^\top]) - \tr(U_k^\top \bE[w w^\top] U_k) \\
		&= \frac13\sum_{i=1}^d \alpha_i^2 - \frac13\sum_{i=1}^k \alpha_i^2
		= \frac13 \sum_{i=k+1}^d \alpha_i^2.
	\end{align*}
    Since it holds $\alpha_i \leq \frac{\alphabar}{i^\gamma}$ 
    for $i=1, \ldots, d$, we have
    \begin{align*}
        \sum_{i=k+1}^d \alpha_i^2
        \leq \alphabar^2 \int_{k}^d \frac{\dd{t}}{t^{2\gamma}}
        \qty(\leq \frac{1}{2\gamma-1} \frac{1}{k^{2\gamma-1}}),
    \end{align*}
    which implies
    \begin{align*}
        \bE[\|U_k U_k^\top w - w\|^2]
        \leq \frac{\alphabar^2}{3} \int_{k}^d \frac{\dd{t}}{t^{2\gamma}}.
    \end{align*}
    Below, we consider approximating $U_k U_k^\top w$ instead of $w$.
    We denote $P_k := U_k U_k^\top \in \bR^{d\times d}$.

    \paragraph{Approximating $P_k w$ using $x_{1:N}$ and $y_{1:N}$}
    Since it holds
    \begin{align*}
        \bE[y_n P_k x_n \mid w]
        = \bE[(w^\top x_n + z_n) P_k x_n]
        = P_k \bE[x_n x_n^\top] w + \bE[z_n] P_k \bE[x_n]
        = P_k w
    \end{align*}
    Moreover, since $x_1, \ldots, x_n$ and $y_1, \ldots, y_n$ are conditionally independent given $w$, we have
    \begin{align*}
        \bE\left[
            \norm{\frac{1}{N} 
                \sum_{n=1}^N y_n P_k x_n - P_k w
            }^2 
            \relmiddle| w
        \right]
        &= \frac{1}{N^2} \sum_{n=1}^N\bE\left[
             \norm{y_n P_k x_n - P_k w}^2 \relmiddle| w
        \right] \\
        &= \frac{1}{N} \bE\left[
             \norm{y_1 P_k x_1 - P_k w}^2 \relmiddle| w
        \right].
    \end{align*}
    Since it holds $y_1 = w^\top x_1 + z_1$, we have
    \begin{align*}
        \bE\left[
             (y_1 P_k x_1)^\top P_k w \relmiddle| w
        \right]
        &= \bE\left[
             y_1 x_1^\top P_k P_k w \relmiddle| w
        \right] \\
        &= \bE\left[
             w^\top P_k P_k w \relmiddle| w
        \right] \\
        &= \norm{P_k w}^2.
    \end{align*}
    Therefore, we have
    \begin{align*}
        \bE\left[
             \norm{y_1 P_k x_1 - P_k w}^2 \relmiddle| w
        \right]
        &= \bE\left[
             \norm{y_1 P_k x_1}^2 \relmiddle| w
        \right]
        - \norm{P_k w}^2.
    \end{align*}
    As for the first term, since $z$ and $x$ are independent, we have
    \begin{align*}
        \bE\left[
             \norm{y_1 P_k x_1}^2 \relmiddle| w
        \right]
        &= \bE\left[
             \norm{(w^\top x_1 + z_1) P_k x_1}^2 \relmiddle| w
        \right] \\
        &= \bE\left[
             (w^\top x_1)^2 \norm{P_k x_1}^2 \relmiddle| w
        \right] + \bE\qty[z_1^2] \bE\qty[\norm{P_k x_1}^2] \\
        &= \frac{d}{d+2} \qty(k\norm{w}^2  + 2 \norm{P_k w}^2) + k.
    \end{align*}
    Here, the last equality follows from the fourth-moment formula for
    $x_1 \sim \unif(\sqrt d\bS^{d-1})$:
    \begin{align*}
        \bE[x_{1,i}x_{1,j}x_{1,a}x_{1,b}]
        = \frac{d}{d+2}
        \qty(\ind{i=j}\ind{a=b}+\ind{i=a}\ind{j=b}+\ind{i=b}\ind{j=a}).
    \end{align*}
    Indeed, since $P_k^2=P_k$ and $\tr(P_k)=k$, we have
    \begin{align*}
        \bE\left[(w^\top x_1)^2\norm{P_kx_1}^2 \relmiddle| w\right]
        &= \bE\left[(w^\top x_1)^2 x_1^\top P_k x_1 \relmiddle| w\right] \\
        &= \sum_{i,j,a,b=1}^d w_i w_j (P_k)_{ab}
            \bE[x_{1,i}x_{1,j}x_{1,a}x_{1,b}] \\
        &= \frac{d}{d+2}\qty(\norm{w}^2\tr(P_k)+2w^\top P_k w) \\
        &= \frac{d}{d+2}\qty(k\norm{w}^2+2\norm{P_kw}^2),
    \end{align*}
    and $\bE[z_1^2]\bE[\norm{P_kx_1}^2]=\tr(P_k)=k$.
    Therefore, we have
    \begin{align*}
        \bE\left[
             \norm{y_1 P_k x_1 - P_k w}^2 \relmiddle| w
        \right]
        &= \frac{d}{d+2} \qty(k\norm{w}^2  + 2 \norm{P_k w}^2) + k
            - \norm{P_k w}^2 \\
        &= \frac{dk}{d+2} \norm{w}^2
            + \frac{d-2}{d+2} \norm{P_k w}^2 + k \\
        &\leq \frac{dk}{d+2} \norm{\alpha}^2
            + \frac{(d-2)k}{d+2} \norm{\alpha}^2 + k \\
        &\leq k (2 \norm{\alpha}^2 + 1).
    \end{align*}
    This implies that
    \begin{align*}
        \bE\left[
            \norm{\frac{1}{N} 
                \sum_{n=1}^N y_n P_k x_n - P_k w
            }^2 
            \relmiddle| w
        \right]
        \leq \frac{k (2 \norm{\alpha}^2 + 1)}{N}.
    \end{align*}

    \paragraph{Evaluating truncation error}
    Next, we upper bound the error between 
    $\frac{1}{N} \sum_{n=1}^N y_n P_k x_n$
    and its data-wise truncation
    $U_k \qty{\frac{1}{N} \sum_{n=1}^N \clip_C (y_n U_k^\top x_n)}$
    for $C > 0$.
    We have
    \begin{align*}
        &\bE \qty[
            \norm{U_k \qty{\frac{1}{N} \sum_{n=1}^N \clip_C (y_n U_k^\top x_n)}
            - \frac{1}{N} \sum_{n=1}^N y_n P_k x_n}^2
        ] \\
        &\qquad = \bE \qty[
            \norm{\qty{\frac{1}{N} \sum_{n=1}^N \clip_C (y_n U_k^\top x_n)}
            - \frac{1}{N} \sum_{n=1}^N y_n U_k^\top x_n}^2
        ] \\
        &\qquad \leq \frac{1}{N} \sum_{n=1}^N\bE \qty[
            \norm{\clip_C (y_n U_k^\top x_n) -  y_n U_k^\top x_n}^2
        ] \\
        &\qquad \leq \bE \qty[
            \norm{\clip_C (y_1 U_k^\top x_1) -  y_1 U_k^\top x_1}^2
        ].
    \end{align*}
    Therefore, it is sufficient to bound 
    $ \bE \qty[ \norm{\clip_C (y_1 U_k^\top x_1) -  y_1 U_k^\top x_1}^2 ]$.
    Let $Y:=y_1$ and $R:=\norm{U_k^\top x_1}$.
    By the definition $\clip_C(v)=\min\{1, C/\norm{v}\}v$, we have
    $\norm{\clip_C(v)-v}=(\norm{v}-C)_+$ for any vector $v$.
    Applying this with $v=Y U_k^\top x_1$, whose norm is $|Y|R$, gives
    \begin{align*}
        \norm{\clip_C (Y U_k^\top x_1) - Y U_k^\top x_1}^2
        = (|Y|R - C)_+^2
        \leq Y^2 R^2 \ind{\{|Y|R>C\}}.
    \end{align*}
    Moreover, the event $\{|Y|R>C\}$ is included in
    $\{Y^2>C\}\cup\{R^2>C\}$.
    Therefore, by Cauchy's inequality,
    \begin{align*}
        &\bE\qty[\norm{\clip_C (Y U_k^\top x_1) - Y U_k^\top x_1}^2 \mid w] \\
        &\qquad\leq \bE\qty[Y^2R^2\ind{\{Y^2>C\}} \mid w]
            + \bE\qty[Y^2R^2\ind{\{R^2>C\}} \mid w] \\
        &\qquad\leq \bE[R^4]^{1/2}
            \bE\qty[Y^4\ind{\{Y^2>C\}} \mid w]^{1/2}
            + \bE\qty[Y^4\mid w]^{1/2}
            \bE\qty[R^4\ind{\{R^2>C\}}]^{1/2}.
    \end{align*}
    Since $\norm{w}\leq\norm{\alpha}$, $Y=w^\top x_1+z_1$ is sub-Gaussian conditionally on $w$
    with a parameter depending only on $\alphabar$ and $\gamma$.
    Also, $R^2=\norm{U_k^\top x_1}^2$ is the squared norm of a $k$-dimensional projection of
    $\unif(\sqrt d\bS^{d-1})$, and hence has a chi-square-type tail.
    Thus, if $C \geq (\|\alpha\|^2 + 1) k$, then
    \begin{align*}
        \bE[R^4] \lesssim k^2,\quad
        \bE\qty[Y^4\ind{\{Y^2>C\}} \mid w] \lesssim \exp(-cC),\quad
        \bE\qty[Y^4\mid w] \lesssim 1,
    \end{align*}
    and
    \begin{align*}
        \bE\qty[R^4\ind{\{R^2>C\}}] \lesssim k^2\exp(-cC),
    \end{align*}
    for a constant $c>0$ depending only on $\alphabar$ and $\gamma$.
    Combining the preceding bounds, we obtain
    \begin{align*}
        \bE\qty[\norm{\clip_C (y_1 U_k^\top x_1) -  y_1 U_k^\top x_1}^2 \mid w]
        \leq \cTrncErr  k \exp(- \cTrncErrExp C),
    \end{align*}
    for constants $\cTrncErr$ and $\cTrncErrExp$ that depend only on $\alphabar$ and $\gamma$.

    \paragraph{Approximating $U_k \qty{\frac{1}{N} \sum_{n=1}^N \clip_C (y_n U_k^\top x_n)}$ via DP-DeepSets}    
	Let $n \in [N]$, $i \in [d]$ and $\tau > 0$.
	\cref{cor:NN-approx-mult-extend} implies that there exists a neural network 
	$\phi_1\in\Phi(L, W, B)$ with
    \begin{align*}
        L \leq 6 + \left\lceil 
            \log_2\frac{4\sqrt{d}\rho}{\varepsilon} 
        \right\rceil, \quad
        W \leq 6, \quad
        B \leq 4 \sqrt{d} \rho.
    \end{align*}
	satisfying $\abs{\phi_1(x_{n,i}, y_n)} \leq \sqrt{d} \rho$ and
	\begin{equation*}
		\abs{\phi_1(x_{n,i}, y_n) - x_{n,i} y_n}
		\leq \begin{cases}
			\varepsilon	         &\quad (|y_n| \leq \rho), \\
			\sqrt{d} \qty(\rho + \abs{y_n})     &\quad (\text{otherwise}).
		\end{cases}
	\end{equation*}
    Therefore, we have
    \begin{align*}
        \bE \qty[\qty{\phi_1(x_{n,i}, y_n) - x_{n,i} y_n}^2]
        &\leq \bE \qty[
            \ind{\qty{\abs{y_n} \leq \rho}} \varepsilon^2
            + \ind{\qty{\abs{y_n} > \rho}} d \qty(\rho + \abs{y_n})^2
        ].
    \end{align*}
    Since it holds $\qty(\abs{y} + \rho)^2 \leq 4y^2 \leq 4y^4/\rho^2$ for $\abs{y} \geq \rho$, we have
    \begin{align*}
        \bE \qty[\qty{\phi_1(x_{n,i}, y_n) - x_{n,i} y_n}^2]
        &\leq \varepsilon^2 + \bE \qty[\frac{4 d y_n^4}{\rho^2}]
        = \varepsilon^2 + \frac{4d}{\rho^2} \bE\qty[y_n^4].
    \end{align*}
    We bound the expectation $\bE\qty[y^4]$.
    It holds
    \begin{align*}
        \bE[y_n^4]
        &= \bE[(w^\top x_n + z_n)^4] \\
        &= \bE[(w^\top x_n)^4] + 6\bE[(w^\top x_n)^2] \bE[z_n^2]
            + \bE[z_n^4] \\
        &= \bE[(w^\top x_n)^4] + 6\bE[(w^\top x_n)^2] + 3.
    \end{align*}
    Moreover, we have
    \begin{align*}
        \bE[(w^\top x_n)^2 \mid w]
        &= w^\top \bE[x_nx_n^\top] w = \frac1d\norm{w}^2 \leq \frac{\cCoefNorm^2}{d}, \\
        \bE[(w^\top x_n)^4 \mid w]
        &= \frac{3}{d(d+2)} \norm{w}^4 
        \leq \frac{3\cCoefNorm^4}{d(d+2)},
    \end{align*}
    where $\cCoefNorm$ is a constant that depends only on $\alpha$.
    This implies that 
    \begin{align*}
        \bE[y_n^4]
        &\leq \frac{3\cCoefNorm^4}{d(d+2)} + \frac{6\cCoefNorm^2}{d} + 3
        \leq \cQuarticY, 
    \end{align*}
    for some constant $\cQuarticY$ that depends only on $\alpha$.
    Taking $\rho := 2\sqrt{\cQuarticY} \cdot \sqrt{d} \varepsilon^{-1}$, we have
    \begin{align*}
        \bE \qty[\qty{\phi_1(x_{n,i}, y_n) - x_{n,i} y_n}^2]
        &\leq \varepsilon^2 + \frac{4 d \cQuarticY}{\rho^2}
        = 2\varepsilon^2.
    \end{align*}

    Next, let $\phi_2 \in \Phi_\mlp(L, W, B)$ with
    \begin{align*}
        L \leq 6 
            + \left\lceil \log_2\frac{8\sqrt{\cQuarticY} d}{\varepsilon^2} \right\rceil, 
            \quad
        W \leq 6 d, \quad
        B \leq 8 \sqrt{\cQuarticY} \cdot d \varepsilon^{-1},
    \end{align*}
    be a neural network defined as
    \begin{align*}
        \phi_2(x_n, y_n)
        = \qty[\phi_1(x_{n,1}, y_n), ~\ldots, ~\phi_1(x_{n, k}, y_n)]^\top 
        \in \bR^{k}
    \end{align*}
    Then, we have
    \begin{align*}
        \bE \qty[\norm{\phi_2(x_n, y_n) - y_n U_k^\top x_n}^2]
        &= \sum_{i=1}^k \bE \qty[\qty{\phi_1(x_{n,i}, y_n) 
            - x_{n,i} y_n}^2]\\
        &\leq 2k \varepsilon^2.
    \end{align*}

    Finally, let $\phi_3 \in \Phi_\mlp(L, W, B)$ with
    $L=1$, $W=d$ and $B=1$ be a neural network defined as
    \begin{align*}
        \phi_3: \bR^k \to \bR^d, \quad v \mapsto U_k v.
    \end{align*}
    Then, we have
    \begin{align*}
        &\bE \qty[\norm{
            \phi_3\qty(
                \frac{1}{N} \sum_{n=1}^N \clip_C (\phi_2(x_n, y_n))
                + \xi
            )
            - U_k \qty{
                \frac{1}{N} \sum_{n=1}^N  \clip_C(y_n U_k^\top x_n)
            }
        }^2] \\
        &\quad \leq 
        \bE \qty[\norm{ U_k \qty{
                \frac{1}{N} \sum_{n=1}^N \clip_C (\phi_2(x_n, y_n))
                - \frac{1}{N} \sum_{n=1}^N  \clip_C(y_n U_k^\top x_n)
            } + U_k \xi
        }^2] \\
        &\quad \leq 
        \bE \qty[\norm{ \qty{
                \frac{1}{N} \sum_{n=1}^N \clip_C (\phi_2(x_n, y_n))
                - \frac{1}{N} \sum_{n=1}^N  \clip_C(y_n U_k^\top x_n)
            } + \xi
        }^2] \\
        &\quad \leq 
        2 \bE \qty[\norm{
            \frac{1}{N} \sum_{n=1}^N \qty{
                \clip_C(\phi_2(x_n, y_n)) - \clip_C(y_n U_k^\top x_n)
            }
        }^2]
        + 2\bE[\norm{\xi}^2] \\
        &\quad \leq 
        \frac{2}{N} \sum_{n=1}^N \bE \qty[
            \norm{\clip_C(\phi_2(x_n, y_n)) - \clip_C(y_n U_k^\top x_n)}^2
        ]
        + 2 k \sdp^2 \\
        &\quad \leq 
        \frac{2}{N} \sum_{n=1}^N \bE \qty[
            \norm{\phi_2(x_n, y_n) - y_n U_k^\top x_n}^2
        ]
        + 2 k \sdp^2 \\
        &\quad \leq 4k \varepsilon^2 + 2k \sdp^2.
    \end{align*}
    Setting $\varepsilon :=\sdp/\sqrt{2}$, we have
    \begin{align*}
        \bE \qty[\norm{
            \phi_3\qty(
                \frac{1}{N} \sum_{n=1}^N \clip_C (\phi_2(x_n, y_n))
                + \xi
            )
            - U_k \qty{
                \frac{1}{N} \sum_{n=1}^N  \clip_C(y_n U_k^\top x_n)
            }
        }^2] \leq 4k\sdp^2.
    \end{align*}

    \paragraph{Finishing the proof}
    If we set
    \begin{align*}
        \phi(p_{1:N}, \xi)
        := \phi_3\qty(
                \frac{1}{N} \sum_{n=1}^N \clip_C (\phi_2(x_n, y_n))
                + \xi
            ),
    \end{align*}
    then we have
    \begin{align*}
        &\bE\qty[\norm{
            \phi(p_{1:N}, \xi)
            - w
        }^2] \\
        &\leq 4 \left\{
            \bE \qty[\norm{
                \phi(p_{1:N}, \xi)
                - U_k \qty{\frac{1}{N} \sum_{n=1}^N  \clip_C(y_n U_k^\top x_n)}
            }^2]
        \right.\\
        &\qquad \left.
            + \bE \qty[\norm{
                U_k \qty{\frac{1}{N} \sum_{n=1}^N  \clip_C(y_n U_k^\top x_n)}
                - \frac{1}{N} \sum_{n=1}^N  y_n P_k x_n
            }^2]
            \right.\\
        &\hspace{50pt} \left.
            + \bE \qty[\norm{
                \frac{1}{N} \sum_{n=1}^N  y_n P_k x_n
                - P_k w
            }^2]
            + \bE \qty[\norm{
                P_k w
                - w
            }^2] 
        \right\}\\
        &\leq 16k\sdp^2
            + 4\cTrncErr  k \exp(- \cTrncErrExp C)
            + \frac{4k (2 \norm{\alpha}^2 + 1)}{N}
            + \frac{4\alphabar^2}{3} \int_{k}^d \frac{\dd{t}}{t^{2\gamma}}.
    \end{align*}
    Let $\cConstruct$ and $\cDpdsErr$ be constants that depend only on $\alphabar$ and $\gamma$.
    If we set
    \begin{align*}
        C = \qty{(\norm{\alpha}^2 + 1) k \vee 
            \frac{1}{\cTrncErrExp} \log \qty(\frac{\epsilon^2 N^2}{\log \delta^{-1}})},
    \end{align*}
    then we have
    \begin{align*}
        \bE\qty[\norm{
            \phi(p_{1:N}, \xi)
            - w
        }^2]
        \leq \cDpdsErr \qty{
            k \qty(\frac{\log \delta^{-1}}{\epsilon^2 N^2} 
                \qty{ k  \vee 
                \log \qty(\frac{\epsilon N}{\log \delta^{-1}})}^2
            + \frac1N) 
            + \int_{k}^d \frac{\dd{t}}{t^{2\gamma}}
        }, 
    \end{align*}
    and $\phi \in \Phi(L, W, B, C, D)$ with
    \begin{align*}
        &L \leq \cConstruct \left\lceil 
            \log \frac{d \cdot \epsilon N}{\log\delta^{-1}}
            \right\rceil , \quad
        W \leq 6 d, \quad
        B \leq 
        \frac{\cConstruct d \cdot \epsilon N}{\log\delta^{-1}}, \\
        & D = k, \quad 
        C = \cConstruct \qty{ k \vee 
            \log \qty(\frac{\epsilon N}{\log \delta^{-1}})},
    \end{align*}
    which completes the proof.
\end{proof}

\subsection{Proof of \cref{thm:upper-bound-hyper-network}}

We finally prove our target statement.
Let 
$
    \phihat := \argmin_{\phi\in\Phi(L, W, B, C, D)} \Rhat(\phi)
$
with
\begin{align*}
    L \simeq \log \frac{d \cdot \epsilon N}{\log\delta^{-1}}, 
    \quad
    W \simeq d, \quad
    B \simeq \frac{d \cdot \epsilon N}{\log\delta^{-1}}, \quad
    D = k, \quad 
    C \simeq k \vee  \log \qty(\frac{\epsilon N}{\log \delta^{-1}}),
\end{align*}
Moreover, let $\phiast \in \Phi(L, W, B, C, D)$ be a network defined in \cref{thm:dp-ds-construct} with the same $L, W, B, C, D$ specified above.
First, we have
\begin{align*}
    &\norm{\phihat(p_{1:N}, \xi) - w}^2 \\
    &\lesssim 
    \norm{\phiast(p_{1:N}, \xi) - w}^2
    + \norm{\phihat(p_{1:N}, \xi) - \phiast(p_{1:N}, \xi)}^2 \\
    &\lesssim 
    \norm{\phiast(p_{1:N}, \xi) - w}^2
    + \frac1M \sum_{m=1}^M
        \norm{\phihat(\pidx{m}_{1:N}, \xiidx{m}) 
            - \phiast(\pidx{m}_{1:N}, \xiidx{m})}^2 \\
    &\quad
    + \abs{
        \frac1M \sum_{m=1}^M
        \norm{\phihat(\pidx{m}_{1:N}, \xiidx{m}) 
            - \phiast(\pidx{m}_{1:N}, \xiidx{m})}^2
        - \bE \qty[\norm{\phihat(p_{1:N}, \xi) - \phiast(p_{1:N}, \xi)}^2]
    } \\
    &\lesssim 
    \norm{\phiast(p_{1:N}, \xi) - w}^2
    + \frac1M \sum_{m=1}^M
        \norm{\phihat(\pidx{m}_{1:N}, \xiidx{m}) 
            - \phiast(\pidx{m}_{1:N}, \xiidx{m})}^2 \\
    &\quad
    + \sup_{\phi \in \Phi} \abs{
        \frac1M \sum_{m=1}^M
        \norm{\phi(\pidx{m}_{1:N}, \xiidx{m}) 
            - \phiast(\pidx{m}_{1:N}, \xiidx{m})}^2
        - \bE_{p_{1:N}, \xi} \qty[\norm{\phi(p_{1:N}, \xi) - \phiast(p_{1:N}, \xi)}^2]
    }.
\end{align*}
The expectation of the first term is bounded in \cref{thm:dp-ds-construct}, i.e., 
\begin{align*}
    \bE\qty[\norm{
        \phiast(p_{1:N}, \xi)
        - w
    }^2]
    \lesssim 
        k \qty(\frac{\log \delta^{-1}}{\epsilon^2 N^2} 
            \qty{ k  \vee 
            \log \qty(\frac{\epsilon N}{\log \delta^{-1}})}^2
        + \frac1N) + \int_{k}^d \frac{\dd{t}}{t^{2\gamma}} =: \Delta_N.
\end{align*}
Moreover, the expectation of the second term is bounded in \cref{lem:upper-bound-empirical-error}:
\begin{align*}
    \bE \qty[\frac{1}{M} \sum_{m=1}^M \norm{
        \phihat(\pidx{m}_{1:N}, \xiidx{m}) 
        - \phiast(\pidx{m}_{1:N}, \xiidx{m})
    }^2] 
    \leq \Otilde\qty(\Delta_N + \frac{d}{N} + \frac{d^2}{\sqrt{M}}).
\end{align*}
Finally, the third term is bounded in \cref{lem:diff-empirical-error-test-error}:
\begin{align*}
    &\bE \qty[\sup_{\phi \in \Phi} \abs{
        \frac1M \sum_{m=1}^M
        \norm{\phi(\pidx{m}_{1:N}, \xiidx{m}) 
            - \phiast(\pidx{m}_{1:N}, \xiidx{m})}^2
        - \bE_{p_{1:N}, \xi} \qty[\norm{\phi(p_{1:N}, \xi) - \phiast(p_{1:N}, \xi)}^2]
    }] \\
    &\hspace{140pt} \lesssim 
    \frac{LW}{\sqrt{M}}
            \sqrt{
            \log\qty{
                d L (B \vee 1) (W+1) (C \vee 1)
                \cdot \epsilon^{-1} \log(\delta^{-1})
            }}\\
    &\hspace{140pt} \lesssim 
    \frac{d \log (d \epsilon N) }{\sqrt{M}}.
\end{align*}
Therefore, we have
\begin{align*}
    &\bE\qty[\norm{\phihat(p_{1:N}, \xi) - w}^2] \\
    &\qquad \leq \Otilde\qty(
    \frac{k \log \delta^{-1}}{\epsilon^2 N^2}
                \qty{ k  \vee
                \log \qty(\frac{\epsilon N}{\log \delta^{-1}})}^2
        + \frac{d}{N} + \frac{d^2 + d \log (\epsilon N)}{\sqrt{M}}
        + \int_{k}^d \frac{\dd{t}}{t^{2\gamma}}
    ),
\end{align*}
which completes the proof.

\section{Lower Bound of DP-SGD}\label[appendix]{app:proof-of-lower-bound-dpsgd}

Here, we provide the proof of \cref{thm:lower-bound-dpsgd}.
\begin{proof}[Proof of \cref{thm:lower-bound-dpsgd}]
We first analyze the update of DP-SGD
when clipping does not occur for the first $t$ steps.
Then, we lower-bound the probability that clipping does not occur.
\paragraph{Explicit update of DP-SGD without clipping}
Recall that we analyze the full-batch DP-SGD update
\begin{align*}
    \omega_t = \omega_{t-1} - \frac{\eta}{N}\qty{
        \sum_{n=1}^N \clip_C(\nabla\scL_n(\omega_{t-1}))
        + \xi_t
    },
    \quad \xi_t \sim \scN(0, s^2 I_d).
\end{align*}
We write $\bar{\xi}_t := \xi_t/N$, so that $\bar{\xi}_t \sim \scN(0, (s^2/N^2) I_d)$.
The loss function can be written as follows:
\begin{align*}
    \scL(\theta) 
    &= \frac{1}{2N} \sum_{i=1}^N (\theta^\top x_i - w^\top x_i - z_i)^2 \\
    &= \frac12 (\theta - w)^\top \qty[\frac1N \sum_{i=1}^N x_i x_i^\top] (\theta - w)
        - \qty[\frac1N \sum_{i=1}^N z_i x_i]^\top (\theta - w)
            + \frac{1}{2N} \sum_{i=1}^N z_i^2 \\
    &= \frac12 (\theta - w)^\top \Sigmahat (\theta - w)
        + \muhat^\top (\theta - w) + \frac{1}{2N} \sum_{i=1}^N z_i^2
\end{align*}
where
\begin{align*}
    \Sigmahat := \frac1N \sum_{i=1}^N x_i x_i^\top, \quad
    \muhat := -\frac1N \sum_{i=1}^N z_i x_i.
\end{align*}
Therefore, the gradient of the loss function is
\begin{align*}
    \nabla\scL(\theta) = \Sigmahat(\theta - w) + \muhat.
\end{align*}
Hence, the $t$-th update of DP-SGD without clipping can be written as
\begin{align*}
    \omega_{t}
    &= \omega_{t-1} - \eta \qty(\nabla\scL(\omega_{t-1}) + \bar{\xi}_t) \\
    &= \omega_{t-1} - \eta \Sigmahat(\omega_{t-1} - w) - \eta (\muhat + \bar{\xi}_t).
\end{align*}
Let $\Delta_t := \omega_t - w$.
Then, we have
\begin{align*}
    \Delta_t
    &= \Delta_{t-1} - \eta \Sigmahat \Delta_{t-1} - \eta (\muhat + \bar{\xi}_t) \\
    &= (I - \eta \Sigmahat) \Delta_{t-1}  - \eta (\muhat + \bar{\xi}_t) .
\end{align*}
Repeating over $t$ steps, we have
\begin{align*}
    \Delta_t
    &= (I - \eta \Sigmahat) \Delta_{t-1} - \eta (\muhat + \bar{\xi}_t) \\
    &= (I - \eta \Sigmahat) \qty{
            (I - \eta \Sigmahat) \Delta_{t-2}
             - \eta (\muhat + \bar{\xi}_{t-1})
        } - \eta (\muhat + \bar{\xi}_t) \\
    &= (I - \eta \Sigmahat)^2 \Delta_{t-2}
        -\qty{
            \eta (I - \eta \Sigmahat) (\muhat + \bar{\xi}_{t-1})
            + \eta (\muhat + \bar{\xi}_t)
        } \\
    &= \cdots \\
    &= (I - \eta \Sigmahat)^t \Delta_{0}
        - \sum_{s=1}^t \eta (I - \eta \Sigmahat)^{t-s} (\muhat + \bar{\xi}_s) \\
    &= (I - \eta \Sigmahat)^t \Delta_{0}
        - (I - (I - \eta \Sigmahat)^t)\Sigmahat^{-1} \muhat
        - \xibar_t,
\end{align*}
where
\begin{align*}
    \xibar_t := \sum_{s=1}^t \eta (I - \eta \Sigmahat)^{t-s} \bar{\xi}_s.
\end{align*}

Conditionally on $x_{1:N}$, 
the variable $\xibar_t$ follows the normal distribution $\scN(0, (s^2/N^2) \Sigmabar_t)$ with
\begin{align*}
    \Sigmabar_t := \sum_{s=1}^t \eta^2 (I - \eta \Sigmahat)^{2(t-s)}
    = \eta \qty[
        I-(I-\eta \widehat{\Sigma})^{2t}
    ] \qty[
        \widehat{\Sigma}(2I-\eta \widehat{\Sigma})
    ]^{-1}.
\end{align*}

\paragraph{Probability of no clipping}
We first lower-bound the probability that clipping does not occur.
We can evaluate the loss associated with $n$-th data point as follows:
\begin{align*}
    \scL_n(\theta)
    &= \frac12 (\theta^\top x_n - w^\top x_n - z_n)^2 \\
    &= \frac12 (\theta - w)^\top \Sigma_n (\theta - w)
        - (z_n x_n)^\top (\theta - w) + \frac{z_n^2}{2},
\end{align*}
where $\Sigma_n := x_n x_n^\top$.
Then, we have
\begin{align*}
    \nabla\scL_n(\theta) = \Sigma_n (\theta - w) - z_n x_n.
\end{align*}
Therefore, the sufficient condition for no clipping is
$\norm{\Sigma_n \Delta_t - z_n x_n} \leq C$ for all $n=1, \ldots, N$ and $t=0, \ldots, T-1$.
We have
\begin{align*}
    \norm{\Sigma_n \Delta_t - z_n x_n}
    \leq \norm{\Sigma_n \Delta_t} + \norm{z_n x_n}.
\end{align*}
Moreover, we have
\begin{align*}
    \norm{\Sigma_n \Delta_t}
    &= \norm{x_n x_n^\top \Delta_t} \\
    &\leq \norm{ (x_n x_n^\top) (I - \eta \Sigmahat)^t \Delta_{0} }
        + \norm{ (x_n x_n^\top) (I - (I - \eta \Sigmahat)^t)\Sigmahat^{-1} \muhat}
        + \norm{(x_n x_n^\top) \xibar_t}.
\end{align*}
We set the clipping size as
\begin{align*}
    C \geq \max\{8 \alphabar, 5 \sqrt{68}\} d.
\end{align*}
It is sufficient to show that
\begin{align*}
\begin{cases}
    \text{(i)} &  \norm*{ (x_n x_n^\top) (I - \eta \Sigmahat)^t \Delta_0 } \leq C/4, \\
    \text{(ii)} & \norm*{ (x_n x_n^\top) (I - (I - \eta \Sigmahat)^t)\Sigmahat^{-1} \muhat} \leq C/4,\\
    \text{(iii)} &\norm*{ (x_n x_n^\top) \xibar_t } \leq C/4, \\
    \text{(iv)} &\norm*{ z_n x_n } \leq C/4, \\
\end{cases}
\end{align*}
for all $t=1, \ldots, T$ and $n = 1, \ldots, N$.

We first prove the inequality (i).
Since $x_n x_n^\top$ is a rank-$1$ matrix and
\begin{align*}
    (x_n x_n)^\top x_n = \|x_n\|^2 x_n = d x_n, 
\end{align*}
its only non-zero (and therefore the maximum) eigenvalue is $d$.
Therefore, we have
\begin{align*}
    \norm{ (x_n x_n^\top) (I - \eta \Sigmahat)^t \Delta_{0} }
    \leq d \|(I - \eta \Sigmahat)^t \Delta_0\|
    \leq d \|\Delta_0\|.
\end{align*}
Moreover, we have
\begin{align*}
    \|\Delta_0\| 
    \leq \sqrt{\frac{\alphabar^2}{1^{2\gamma}} 
        + \cdots + \frac{\alphabar^2}{d^{2\gamma}}}
    \leq \alphabar \sqrt{1 + \int_{1}^d x^{-2\gamma} \dd{x}}
    \leq 2\alphabar.
\end{align*}
Hence, we have
\begin{align*}
    \norm{ (x_n x_n^\top) (I - \eta \Sigmahat)^t \Delta_{0} }
    \leq 2\alphabar d
    \leq \frac{C}{4}.
\end{align*}

Next, we prove (ii).
We first define the event used to control the inverse of $\Sigmahat$.
By \cref{lem:eigenvalue-concentration}, with probability $1 - \frac{1}{17}\cdot\frac{1}{2}$,
all of the eigenvalues are in $[1 - \zeta_N, 1 + \zeta_N]$, where
\begin{align*}
    \zeta_N := \cEigen
        \qty{
            \qty(\sqrt{\frac dN} + \sqrt{\frac{\log(68)}{N}})
            + \qty(\sqrt{\frac dN} + \sqrt{\frac{\log(68)}{N}})^2
        }.
\end{align*}
We denote this event by $\Eeigen$.
In particular, if
$N \geq \max\{144\cEigen^2, 12\cEigen\}(\sqrt{d} + \sqrt{\log(68)})^2$,
then $\zeta_N \leq 1/6$.
We upper-bound $\|\muhat\|$.
By definition, we have
\begin{align*}
    \|\muhat\|_2^2 
    = \frac{1}{N^2} \sum_{i,j=1}^N z_i z_j \cdot x_i^\top x_j
\end{align*}
Taking expectation, we have
\begin{align*}
    \bE[\|\muhat\|_2^2] 
    &= \frac{1}{N^2} \sum_{i=1}^N \bE[z_i^2 \cdot \|x_i\|^2] 
    = \frac{1}{N^2} \sum_{i=1}^N \bE[z_i^2] \cdot \bE[\|x_i\|^2] \notag \\
    &= \frac{1}{N^2} \sum_{i=1}^N \bE[x_i^\top x_i] 
    = \frac{1}{N^2} \sum_{i=1}^N \bE[\tr(x_i x_i^\top)] \notag \\
    &= \frac{\tr(I)}{N} =\frac{d}{N}.
\end{align*}
Therefore, Markov's inequality implies
$
    \|\muhat\|_2 \leq \sqrt{\frac{68d}{N}}
$
with probability $1 - \frac{1}{17}\cdot \frac{1}{4}$.
Therefore, we have 
$\norm{\muhat} \leq \sqrt{\frac{68d}{N}}$.
In particular, if $N \geq d$, we have
$\norm{\muhat} \leq \sqrt{68}$.
Therefore, similarly to (i), we have
\begin{align*}
    \norm{ (x_n x_n^\top) (I - (I - \eta \Sigmahat)^t)\Sigmahat^{-1} \muhat}
    &\leq d \norm{(I - (I - \eta \Sigmahat)^t)\Sigmahat^{-1} \muhat} \\
    &\leq \frac{6}{5} d \norm{\muhat}
    \leq \frac{6\sqrt{68}}{5} d
    \leq \frac{C}{4}.
\end{align*}
on the event $\Eeigen$.

Let us prove (iii) next.
On the same event $\Eeigen$, we have
\begin{align*}
    \norm{\Sigmabar_t}_{\rmop}
    \leq \frac{\eta ( 1 - (1 - \eta\cdot\frac56)^{2t})}
        {\frac56 \cdot \qty(2 - \eta \cdot \frac76)}
    \leq \frac{36\eta}{5(12 - 7\eta)} < 3, 
\end{align*}
which implies $x_n^\top \Sigmabar_t x_n \leq 3 \norm{x_n}^2$.
Since it holds $
    x_n^\top \xibar_t \mid x_n
    \sim \scN(0, (s^2/N^2) x_n^\top \Sigmabar_t x_n)
$, we have
\begin{align*}
    \bP \left(
        \abs{x_n^\top \xibar_t} \leq \frac{3s}{N} \sqrt{2 \log(136 NT)} \norm{x_n}
        \relmiddle| x_n
    \right)
    &\geq \bP \left(
        \abs{x_n^\top \xibar_t} \leq \frac{s}{N} \sqrt{6 \log(136 NT)} \norm{x_n}
        \relmiddle| x_n
    \right) \\
    &\geq \bP \left(
        \abs{x_n^\top \xibar_t} \leq \sqrt{\frac{2s^2}{N^2} x_n^\top \Sigmabar_t x_n \log(136 NT)}
        \relmiddle| x_n
    \right) \\
    &\geq 1 - \frac{1}{136 NT}.
\end{align*}
Union bound over $t=1, \ldots, T$ and $n=1, \ldots, N$ implies that
\begin{align*}
    \abs{x_n^\top \xibar_t} 
    \leq \frac{3s}{N} \sqrt{2 \log(136 NT)} \norm{x_n},
\end{align*}
for all $t=1, \ldots, T$ and $n=1, \ldots, N$, 
with probability $1 - \frac{1}{17}\cdot \frac{1}{8}$.
Under this event, it holds
\begin{align*}
    \norm{x_n x_n^\top \xibar_t}
    &= \abs{x_n^\top \xibar_t} \norm{x_n}
    \leq \frac{3s}{N} \sqrt{2 \log(136 NT)} \|x_n\|^2 \\
    &\leq 3 d \sconst C
        \frac{\sqrt{2T \log(136 NT) \log (\delta^{-1})}}{N\epsilon}, 
\end{align*}
for all $n=1, \ldots, N$ and $t=1, \ldots, T$.
If $N \geq \sqrt{2} N_0 \log N_0$ with
\begin{align*}
    N_0 = 12 \sconst d \sqrt{2T \log(136T) \epsilon^{-2} \log (\delta^{-1})}, 
\end{align*}
we have
\begin{align*}
    \norm{x_n x_n^\top \xibar_t} 
    &\leq \frac{C}{4} \frac{
        \sqrt{
            (12d \sconst)^2 \cdot 2T \log(136 T) \epsilon^{-2} \log (\delta^{-1})
            + (12 d \sconst)^2 \cdot 2T \log (N) \epsilon^{-2} \log (\delta^{-1})
        }
    }{N} \\
    &\leq \frac{C}{4} \frac{\sqrt{N_0^2 + (12 d \sconst)^2 \cdot 2T \log (N_0) \epsilon^{-2}\log(\delta^{-1})}}{\sqrt{2} N_0 \log N_0} \\
    &\leq \frac{C}{4} \frac{\sqrt{N_0^2 + N_0^2 \log (N_0)}}{\sqrt{2} N_0 \log N_0}
    \leq \frac{C}{4},
\end{align*}
for all $t=1, \ldots, T$ and $n=1, \ldots, N$.

Finally, we prove (iv).  
First, we have
$
    \norm{z_n x_n} = \abs{z_n} \norm{x_n} = \sqrt{d} \abs{z_n}
$.
Since $z_n \sim \scN(0, 1)$, we have
\begin{align*}
    \norm{z_n x_n} \leq \sqrt{2d \log\frac{2}{(272 N)^{-1}}}, 
\end{align*}
with probability $1 - \frac{1}{17} \cdot \frac{1}{16 N}$.
Taking a union bound for $n= 1, \ldots, N$, we have
\begin{align*}
    \norm{z_n x_n} \leq \sqrt{2d \log(544)} \leq 3d, 
\end{align*}
for all $n = 1, \ldots, N$, 
with probability $1 - \frac{1}{17} \cdot \frac{1}{16}$.

Summarizing the bounds above,
if $N \geq \max\{\sqrt{2} N_0 \log N_0, d\}$ and $C \geq \max\{8 \alphabar, 5 \sqrt{68}\} d$,
we have the following with probability 
$1 - \frac{1}{17} \qty( \frac14 + \frac18 + \frac{1}{16} )$: 
\begin{align*}
    \sup_{t=1, \ldots, T;~n=1, \ldots, N} \norm{\nabla \scL_n(\omega_{t-1})} \leq C,
\end{align*}
which implies that no clipping occurs.

\paragraph{Lower bound of $\|\Delta_T\|$}
We first condition on $x_{1:N}$ and $\Delta_0$.
Recall the following explicit form of $\Delta_T$:
\begin{align*}
    \Delta_T 
    &= (I - \eta \Sigmahat)^T \Delta_{0}
        - (I - (I - \eta \Sigmahat)^T)\Sigmahat^{-1} \muhat
        - \xibar_T, \\
    &= (I - \eta \Sigmahat)^T \Delta_{0}
        - B_T \muhat
        - \xibar_T.
\end{align*}
where $\xibar_T \sim \scN(0, (s^2/N^2) \Sigmabar_T)$ with
$\Sigmabar_T = \eta \qty[
        I-(I-\eta \widehat{\Sigma})^{2T}
    ] \qty[
        \widehat{\Sigma}(2I-\eta \widehat{\Sigma})
    ]^{-1}$, 
and $B_T := (I - (I - \eta \Sigmahat)^T)\Sigmahat^{-1}$.
Since it holds
\begin{align*}
    \bE \left[\muhat \muhat^\top \relmiddle| x_{1:N} \right]
    &= \frac{1}{N^2} \bE \left[
        \sum_{n=1}^N z_n^2 x_n x_n^\top
        \relmiddle| x_{1:N}
    \right]
    = \frac{1}{N^2} 
        \sum_{n=1}^N x_n x_n^\top
    = \frac{1}{N} \Sigmahat,
\end{align*}
we have
\begin{align*}
    \qty [- B_T \muhat - \xibar_T] \sim \scN(0, S)
    \quad\text{where}\quad
    S = \frac{1}{N} B_T\Sigmahat B_T^\top + \frac{s^2}{N^2} \Sigmabar_T.
\end{align*}
Therefore, using \cref{lem:lower-bound-uncentered-gaussian}, we have
\begin{align*}
    \norm{\Delta_T}^2 
    \geq \frac12 Q_T,
\end{align*}
with probability $\frac{1}{17}$.
Here, we define
\begin{align*}
    Q_T
    := \norm{(I - \eta \Sigmahat)^T \Delta_{0}}^2
        + \frac{1}{N}\tr(B_T\Sigmahat B_T^\top)
        + \frac{s^2}{N^2} \tr(\Sigmabar_T).
\end{align*}

Next, we lower-bound the right-hand side with constant probability
on $x_{1:N}$ and $\Delta_0$.
We first recall that, under the event $\Eeigen$, we have
$\lambdamax \in [1, 1+\zeta_N]$ and $\lambdamin \in [1-\zeta_N, 1]$.
We also recall that, if
$N \geq \max\{144 \cEigen^2, 12 \cEigen\}(\sqrt{d} + \sqrt{\log(68)})^2$, 
then $\zeta_N \leq 1/6$.

We consider the following two cases separately.
\begin{itemize}\setlength{\leftskip}{-8mm}
    \item \textbf{Case (i): $\eta \in [6/7, 1)$}.
    In this case, for $T \geq 1$, we have
    \begin{align*}
        \frac{1}{N}\tr(B_T\Sigmahat B_T^\top)
        &= \frac{1}{N}\sum_{i=1}^d \frac{\qty{1 - (1 - \eta\lambda_i)^T}^2}{\lambda_i} \\
        &\geq \frac{6}{7N}\sum_{i=1}^d \qty{1 - (1 - \eta\lambda_i)^T}^2
        \geq \frac{150}{343}\frac{d}{N},
    \end{align*}
    where we used $\lambda_i \leq 7/6$ and $\abs{1 - \eta\lambda_i} \leq 2/7$ on $\Eeigen$.
    Moreover,
    \begin{align*}
        \tr(\Sigmabar_T)
        &= \tr\qty(\sum_{s=1}^T \eta^2 (I - \eta \Sigmahat)^{2(T-s)}) \\
        &\geq \tr\qty(\eta^2 I) = \eta^2 d \geq \frac{36}{49} d, 
    \end{align*}
    which implies
    \begin{align*}
        Q_T
        \geq \frac{150}{343}\frac{d}{N}
            + \frac{36}{49} \frac{d s^2}{N^2}
        \geq \frac{150}{343}\frac{d}{N}
            + \frac{36}{49} \frac{d s^2}{T N^2}.
    \end{align*}
    
    \item \textbf{Case (ii): $\eta \in (0, 6/7)$}.
    Let $0\leq \lambdamin = \lambda_1 \leq \cdots \leq \lambda_d = \lambdamax$.
    Then, we have
    \begin{align*}
        \tr(\Sigmabar_T)
        = \tr\qty(\sum_{s=1}^T \eta^2 (I - \eta \Sigmahat)^{2(T-s)})
        = \eta^2 \sum_{k=0}^{T-1} \sum_{i=1}^d (1 - \eta \lambda_i)^{2k}.
    \end{align*}
    Since it holds 
    $\eta \lambdamax < \frac{6}{7}\cdot \frac{7}{6} = 1$ 
    under the event $\Eeigen$, we have
    \begin{align*}
        \tr(\Sigmabar_T)
        &\geq \eta^2 \sum_{k=0}^{T-1} \sum_{i=1}^d (1 - \eta \lambdamax)^{2k} \\
        &= d \eta \frac{1 - (1 - \eta \lambdamax)^{2T}}
            {\lambdamax(2 - \eta\lambdamax)}\\
        &= d \eta \lambdamax \frac{1 - (1 - \eta \lambdamax)^{2T}}
            {\lambdamax^2(2 - \eta\lambdamax)}.
    \end{align*}
    We upper-bound the factor $\lambdamax^2 (2 - \eta \lambdamax)$.
    The function $\lambda \mapsto \lambda^2 (2 - \eta \lambda)$ 
    monotonically increases in $\lambda \in \qty[0, \frac{4}{3\eta}]$.
    Since $\eta^{-1} > 7/6$, it holds $\frac{4}{3\eta} > \frac{14}{9} > \frac{7}{6}$. 
    Using $\lambdamax \leq \frac76$, we have
    \begin{align*}
        \lambdamax^2 (2 - \eta \lambdamax) 
        \leq \qty(\frac76)^2 \qty(2 - \eta \cdot \frac76)
        \leq \frac{49}{18}.
    \end{align*}
    This implies that
    \begin{align*}
        \tr(\Sigmabar_T)
        \geq \frac{18}{49} d \eta \lambdamax \qty[1 - (1 - \eta \lambdamax)^{2T}].
    \end{align*}
    Moreover, we have
    \begin{align*}
        \norm{(I - \eta \Sigmahat)^T \Delta_{0}}^2
        \geq (1 - \eta\lambdamax)^{2T} \norm{\Delta_{0}}^2.
    \end{align*}
    Also, we have
    \begin{align*}
        \frac{1}{N}\tr(B_T\Sigmahat B_T^\top)
        &= \frac{1}{N}\sum_{i=1}^d \frac{\qty{1 - (1 - \eta\lambda_i)^T}^2}{\lambda_i} \\
        &\geq \frac{150}{343}\frac{d}{N}
            \qty{1 - (1 - \eta\lambdamax)^T}^2.
    \end{align*}
    Here, we used $1 - (1 - \eta\lambda_i)^T \geq \frac57\qty{1 - (1 - \eta\lambdamax)^T}$,
    which follows from $\lambda_i \geq \frac57\lambdamax$ on $\Eeigen$.
    Combining the preceding two displays and using
    $a(1-y)^2 + by^2 \geq \frac12\min\{a,b\}$ for $a,b>0$ and $y\in[0,1]$, we obtain
    \begin{align*}
        \norm{(I - \eta \Sigmahat)^T \Delta_{0}}^2
        + \frac{1}{N}\tr(B_T\Sigmahat B_T^\top)
        \geq \frac12\min\qty{\frac{150}{343}\frac{d}{N}, \norm{\Delta_0}^2}.
    \end{align*}
    Therefore, we have
    \begin{align*}
        &\hspace{-8mm}\norm{(I - \eta \Sigmahat)^T \Delta_{0}}^2
            + \frac{s^2}{N^2} \tr(\Sigmabar_T) \\
        &\geq \frac{18}{49} \frac{d s^2}{T N^2} T\eta \lambdamax \qty[1 - (1 - \eta \lambdamax)^{2T}]
            + \norm{\Delta_{0}}^2 (1 - \eta\lambdamax)^{2T}.
    \end{align*}
    Since it holds $\eta \lambdamax \in [0, 1]$, we can apply
    \cref{lem:sum-of-noise-and-bias} with $x \leftarrow \eta \lambdamax$, 
    $A \leftarrow \frac{18}{49} \frac{d s^2}{T N^2}$, and $B \leftarrow \norm{\Delta_0}^2$.
    Therefore, we have
    \begin{align*}
        \hspace{-8mm}\norm{(I - \eta \Sigmahat)^T \Delta_{0}}^2
            + \frac{s^2}{N^2} \tr(\Sigmabar_T)
        &\geq \min\qty{
            \frac{1}{40}\cdot \frac{18}{49} \frac{d s^2}{T N^2},
            \frac{1}{80} \norm{\Delta_{0}}^2
        }.
    \end{align*}
    Combining the two lower bounds in Case (ii), and using $Q_T \geq A$ and $Q_T \geq B$ to obtain
    $Q_T \geq (A+B)/2$ when the two lower bounds above control $A$ and $B$, we have
    \begin{align*}
        Q_T
        \gtrsim \min\qty{\frac{d}{N}, \norm{\Delta_0}^2}
            + \min\qty{\frac{d s^2}{T N^2}, \norm{\Delta_0}^2}.
    \end{align*}
\end{itemize}
Combining (i) and (ii), we have
\begin{align*}
    Q_T
    \gtrsim \min\qty{\frac{d}{N}, \norm{\Delta_0}^2}
        + \min\qty{\frac{d s^2}{T N^2}, \norm{\Delta_0}^2}.
\end{align*}
Additionally, with probability $1 - \frac{1}{17}\cdot\frac{1}{32}$, 
it holds $\norm{\Delta_0} \geq \frac{\alphaul}{544}$, and
\begin{align*}
    Q_T
    \gtrsim \frac{d}{N} + \min\qty{\frac{d s^2}{T N^2}, 1},
\end{align*}
where we used $N\geq d$.

Therefore, we have
\begin{align}\label{eq:lower-bound-DeltaT-explicit}
    \norm{\Delta_T}^2 
    \gtrsim \frac{d}{N} + \min\qty{\frac{d s^2}{T N^2}, 1}.
\end{align}

\paragraph{Finishing the proof}
The conditions imposed on $N$ and $C$ above are summarized as
\begin{align*}
    N \geq \max\qty{d, \sqrt{2}N_0 \log N_0, 
        \max\{144 \cEigen^2, 12 \cEigen\}(\sqrt{d} + \sqrt{\log(68)})^2}, 
\end{align*}
with $N_0 = 12 \sconst d \sqrt{2T \log(136T) \epsilon^{-2} \log (\delta^{-1})}$
and $C \geq \max\{8 \alphabar, 5 \sqrt{68}\} d$.
Under these conditions, 
we have \eqref{eq:lower-bound-DeltaT-explicit}
with probability
\begin{align*}
    \frac{1}{17} \qty(1 - \frac{1}{2} - \frac{1}{4} - \frac{1}{8} - \frac{1}{16} - \frac{1}{32}) \geq \frac{1}{17} \cdot \frac{1}{32}.
\end{align*}
Therefore, we have
\begin{align*}
    \bE[\|\omega_T - w\|^2]
    &\gtrsim \frac{d}{N} + \min\qty{\frac{d s^2}{T N^2}, 1} \\
    &\gtrsim 
    \frac{d}{N} + \min\qty{d \frac{C^2 \log(\delta^{-1})}{\epsilon^2 N^2}, 1} \\
    &\gtrsim 
    \frac{d}{N} + \min\qty{\frac{d^3 \log(\delta^{-1})}{\epsilon^2 N^2}, 1}, 
\end{align*}
which completes the proof.
\end{proof}

\section{Auxiliary Lemmas}

The following lemma is about the concentration of eigenvalues of the empirical covariance matrix.
\begin{lemma}\label{lem:eigenvalue-concentration}
    Let $N\in\bZ_{>0}$, 
    and $x_1, \ldots, x_N \sim \unif(\sqrt{d}\bS^{d-1})$ be i.i.d. random variables.
    We define $\Sigmahat := \frac1N \sum_{n=1}^N x_n x_n^\top$, 
    and let $\lambda_1, \ldots, \lambda_d$ be the eigenvalues of $\Sigmahat$.
    Suppose that $N > d$.
    Then, we have the following two facts:
    \begin{itemize}
        \item For any $\delta \in (0, 1)$, it holds
            \begin{align*}
                \sup_{i=1, \ldots, d} \abs{\lambda_i - 1}
                \leq \cEigen
                \qty{
                    \qty(\sqrt{\frac dN} + \sqrt{\frac{\log(2/\delta)}{N}})
                    + \qty(\sqrt{\frac dN} + \sqrt{\frac{\log(2/\delta)}{N}})^2
                }.
            \end{align*}
            with probability $1 - \delta$, 
            where $\cEigen$ is an absolute constant.
        \item There exists a constant $\cEigenMin$ such that
            \begin{align*}
                \bP(\lambdamin(\Sigmahat) \leq t) 
                \leq 2\qty(K_{N,d} \sqrt{t})^{N-d+1}, 
            \end{align*}
            for all $t \in (0, K_{N,d}^{-2}]$ where 
            $K_{N,d} := \cEigenMin \frac{\sqrt{N}}{\sqrt{d}}
                \frac{\sqrt{d} + \sqrt{2\log(2N)}}{\sqrt{N} - \sqrt{d}}$.
    \end{itemize}
\end{lemma}
\begin{proof}
    We first prove the first argument.
    Let $A \in \bR^{N\times d}$ be the matrix whose $n$-th row is $x_n^\top$.
    Then $\Sigmahat = A^\top A/N$.
    The random vector $x_n \sim \unif(\sqrt d\bS^{d-1})$ is mean zero and isotropic:
    by rotational invariance, $\bE[x_nx_n^\top]=aI_d$ for some $a>0$, and taking traces gives
    $ad=\bE\|x_n\|^2=d$, hence $a=1$.
    Moreover, $x_n$ is sub-Gaussian with an absolute sub-Gaussian norm.
    Indeed, for every $u\in\bS^{d-1}$, $\langle x_n,u\rangle$ has the same distribution as
    $\sqrt d\,\theta_1$, where $\theta\sim\unif(\bS^{d-1})$; the standard concentration inequality
    on the sphere gives
    \begin{align*}
        \bP\qty(\abs{\langle x_n,u\rangle}>t)
        \leq 2\exp(-c t^2),\quad t\geq 0,
    \end{align*}
    for an absolute constant $c>0$.
    Thus $\sup_{u\in\bS^{d-1}}\|\langle x_n,u\rangle\|_{\psi_2}$ is bounded by an absolute constant.

    Theorem 4.6.1 of \citet{vershynin2020high}, applied to the matrix $A$ with independent,
    mean-zero, isotropic, sub-Gaussian rows, implies that for every $t\geq 0$, with probability
    at least $1-2\exp(-t^2)$,
    \begin{align*}
        \norm{\Sigmahat-I_d}_{\rmop}
        = \norm{\frac{A^\top A}{N}-I_d}_{\rmop}
        \leq C
        \qty{
            \frac{\sqrt d+t}{\sqrt N}
            + \qty(\frac{\sqrt d+t}{\sqrt N})^2
        },
    \end{align*}
    where $C>0$ is an absolute constant.
    Taking $t=\sqrt{\log(2/\delta)}$ yields the bound with probability at least $1-\delta$.
    Finally, since $\Sigmahat-I_d$ is symmetric,
    \begin{align*}
        \sup_{i=1,\ldots,d}\abs{\lambda_i-1}
        = \norm{\Sigmahat-I_d}_{\rmop},
    \end{align*}
    which proves the first argument after absorbing constants into $\cEigen$.
    
    We prove the second argument.
    Let $u_n := x_n /\sqrt{d}$.
    Then, we have $u_n \sim \unif(\bS^{d-1})$.
    Let $r_1, \ldots, r_N$ be independent random variables, independent of $u_1, \ldots, u_N$,
    whose common distribution is the norm of a standard Gaussian vector in $\bR^d$.
    By the polar decomposition of a Gaussian vector, $g_n:=r_nu_n$ are i.i.d. $\scN(0,I_d)$.
    Let $U := [u_1, \ldots, u_N]$
    and $G := [g_1, \ldots, g_N] = UD$
    with $D = \diag(r_1, \ldots, r_N)$.
    Moreover, we have $X = \sqrt{d} U$ and $\Sigmahat = \frac{d}{N} U U^\top$, which yields
    \begin{align*}
        \lambdamin(\Sigmahat) 
        = \frac{d}{N} s_{\min}(U)^2, 
    \end{align*}
    where $s_{\min}$ is the smallest singular value.

    Let $M := \max_{n=1, \ldots, N} r_n$.
    Then, we have $\bP(r_n \geq \sqrt{d} + t) \leq \rme^{-t^2/2}$.
    Therefore, we have
    \begin{align*}
        \bP(M \leq \sqrt{d} + \sqrt{2\log(2N)}) \geq \frac12.
    \end{align*}
    On the event $\scE := \{M \leq \sqrt{d} + \sqrt{2\log(2N)}\}$, 
    we have
    \begin{align*}
        s_{\min}(G)
        = s_{\min}(UD) 
        \leq \|D\|_{\rmop} s_{\min}(U)
        \leq (\sqrt{d} + \sqrt{2\log(2N)})s_{\min}(U).
    \end{align*}
    Therefore, we have
    \begin{align*}
        \{s_{\min}(U) \leq u\} \cap \scE
        \subset \{s_{\min}(G) \leq (\sqrt{d} + \sqrt{2\log(2N)})u\}.
    \end{align*}
    Since $\{s_{\min}(U)\leq u\}$ depends only on the directions $u_1,\ldots,u_N$,
    whereas $\scE$ depends only on the independent radii $r_1,\ldots,r_N$, these two events are independent.
    This implies
    \begin{align*}
        \bP(s_{\min}(U) \leq u)
        &\leq \frac{\bP(s_{\min}(G) \leq (\sqrt{d} + \sqrt{2\log(2N)})u)}{\bP(\scE)} \\
        &\leq 2 \bP(s_{\min}(G) \leq (\sqrt{d} + \sqrt{2\log(2N)})u).
    \end{align*}

    Since $G$ is a random matrix whose entries independently follow $N(0, 1)$, the standard lower-tail bound
    for the smallest singular value of rectangular Gaussian matrices gives, for $\epsilon > 0$,
    \begin{align*}
        \bP(s_{\min}(G) \leq \epsilon(\sqrt{N} - \sqrt{d}))
        \leq (\cEigenMin \epsilon)^{N-d+1}.
    \end{align*}
    Therefore, we have
    \begin{align*}
        \bP(s_{\min}(U) \leq u)
        \leq 2 \qty(
            \cEigenMin
            \frac{(\sqrt{d} + \sqrt{2\log(2N)})u}{\sqrt{N}-\sqrt{d}}
        )^{N-d+1}.
    \end{align*}
    Hence, we have
    \begin{align*}
        \bP(\lambdamin(\Sigmahat)\leq t)
        = \bP\qty(s_{\min}(U) \leq \sqrt{\frac{Nt}{d}})
        \leq 2 (K_{N,d} \sqrt{t})^{N-d+1},
    \end{align*}
    which completes the proof.
\end{proof}

The following lemma provides a constant lower bound on the probability that the norm of an uncentered Gaussian vector exceeds a specified threshold.
\begin{lemma}\label{lem:lower-bound-uncentered-gaussian}
    Let $p \in \bR^d$ and $x \sim N(0, S)$ for some positive definite matrix $S \in \bR^{d\times d}$.
    Then, we have
    \begin{align*}
        \bP\qty[\norm{p + x}^2 \geq \frac{\norm{p}^2 + \tr S}{2}] 
        \geq \frac{1}{17}.
    \end{align*}
\end{lemma}
\begin{proof}
    First, we evaluate the expectation and the variance of $\norm{p + x}$.
    We have
    \begin{align*}
        \bE\qty[
            \norm{p + x}^2
        ] 
        &= \bE \qty[\norm{p}^2 + 2 p^\top x + \norm{x}^2] \\
        &= \norm{p}^2 + \bE[\norm{x}^2] \\
        &= \norm{p}^2 + \tr S =: M.
    \end{align*}
    Moreover, we have
    \begin{align*}
        \Var\qty(\norm{p + x}^2)
        &= \Var\qty(2 p^\top x + \norm{x}^2) \\
        &= 4\Var(p^\top x) + \Var(x^\top x) + 4\Cov(p^\top x, x^\top x) \\
        &= 4 p^\top S p + 2 \tr(S^2).
    \end{align*}
    
    We upper-bound the variance using $M^2$.
    As for the first term, we have
    \begin{align*}
        p^\top S p \leq \lambdamax(S) \norm{p}^2 \leq \tr(S) \norm{p}^2.
    \end{align*}
    Moreover, the second term can be upper-bounded as
    \begin{align*}
        \tr(S^2)
        = \sum_{i=1}^d \lambda_i^2
        \leq \qty(\sum_{i=1}^d \lambda_i)^2
        = \tr(S)^2.
    \end{align*}
    Therefore, we have
    \begin{align*}
        \Var\qty(\norm{p + x}^2)
        &= 4 p^\top S p + 2 \tr(S^2) \\
        &\leq 4\tr(S) \norm{p}^2 + 2 \tr(S)^2 \\
        &\leq \tr(S) \qty(4\norm{p}^2 + 2 \tr(S)) \\
        &\leq 4 M^2.
    \end{align*}

    Now, Cantelli's inequality implies
    \begin{align*}
        \bP \qty(\norm{p + x}^2 - M \leq - \frac{M}{2})
        &\leq \frac{\Var\qty(\norm{p + x}^2)}{\Var\qty(\norm{p + x}^2) + (M/2)^2} \\
        &= 1 - \frac{M^2}{4\Var\qty(\norm{p + x}^2) + M^2} \\
        &\leq 1- \frac{M^2}{16M^2 + M^2} = \frac{16}{17}.
    \end{align*}
    This yields
    \begin{align*}
        \bP \qty(\norm{p + x}^2 \geq \frac{M}{2}) \geq \frac{1}{17}, 
    \end{align*}
    which completes the proof.
\end{proof}

The following lemma is used in the proof of \cref{thm:lower-bound-dpsgd} 
to lower-bound the sum of two terms;
one from the DP noise and the other from the bias of the initial point.
\begin{lemma}\label{lem:sum-of-noise-and-bias}
    Let $A, B > 0$ and $T \in \bZ_{\geq 0}$.
    We define a function $f: [0, 1] \to \bR$ as 
    $f(x) = AT x (1 - (1-x)^{2T}) + B(1-x)^{2T}$.
    Then, we have
    \begin{align*}
        \min_{x\in[0, 1]} f(x) 
        \geq \min\qty{\frac{A}{40}, \frac{B}{80}}.
    \end{align*}
\end{lemma}
\begin{proof}
    If $T=0$, then $f(x) = B$, whose minimum is $B$.

    We consider the case of $T \geq 1$.
    Let $R := B / A$ and $y := (1 - x)^{2T}$.
    Then, we have
    \begin{align*}
        \frac{f(x)}{A} = Tx (1-y) + Ry.
    \end{align*}
    
    Suppose that $1/2 \leq x \leq 1$.
    Then, we have $y \leq 2^{-2T} \leq 1/4$, which yields
    \begin{align*}
        \frac{f(x)}{A} 
        \geq Tx(1-y) 
        \geq T \cdot \frac12 \cdot \frac34 
        = \frac{3T}{8}.
    \end{align*}
    
    Next, we consider $0 \leq x \leq 1/2$.
    Since it holds $-2x \leq \log(1-x) \leq -x$, we have
    \begin{align*}
        \rme^{-4Tx} \leq y \leq \rme^{-2Tx}.
    \end{align*}
    Therefore, we have
    \begin{align*}
        \frac{f(x)}{A} \geq  Tx (1-\rme^{-2Tx}) + R \rme^{-4Tx}.
    \end{align*}

    We define $M := \min\qty{T, \log(1 + R)}$, 
    and consider the following two cases separately.
    \begin{itemize}\setlength{\leftskip}{-8mm}
        \item \textbf{Case (i): $M \leq 1$.}
        We first consider the case of $2Tx \leq 1$.
        We have
        \begin{align*}
            \frac{f(x)}{A} 
            \geq  R \rme^{-4Tx} 
            \geq \rme^{-2} R
            \geq \rme^{-2} \log(1 + R)
            \geq \rme^{-2} M.
        \end{align*}
        On the other hand, if $2Tx > 1$ we have
        \begin{align*}
            \frac{f(x)}{A} 
            \geq Tx (1-\rme^{-2Tx}) 
            \geq \frac12 (1 - \rme^{-1})
            \geq \frac12 (1 - \rme^{-1}) M.
        \end{align*}
        
        \item \textbf{Case (ii): $M > 1$.}
        If $Tx \geq M/8~(> 1/8)$, we have
        \begin{align*}
            \frac{f(x)}{A} 
            \geq Tx (1-\rme^{-1/4})
            \geq \frac{1-\rme^{-1/4}}{8} M. 
        \end{align*}
        Next, we consider the case of $Tx < M/8$.
        Since it holds $M \leq \log(1 + R)$ and $M > 1$, we have
        \begin{align*}
            R \geq \rme^M - 1
            = \rme^M \qty(1 - \rme^{-M})
            \geq \rme^M \qty(1 - \rme^{-1}).
        \end{align*}
        Combining with $4Tx < M/2$, we have
        \begin{align*}
            \frac{f(x)}{A} 
            \geq  R \rme^{-4Tx} 
            \geq \rme^M\qty(1 - \rme^{-1}) \cdot \rme^{-M/2}
            = \qty(1 - \rme^{-1}) \rme^{M/2}.
        \end{align*}
        Since it holds $\rme^{M/2} \geq M/2$, we have
        \begin{align*}
            \frac{f(x)}{A} 
            \geq \frac{1 - \rme^{-1}}{2} M.
        \end{align*}
    \end{itemize}
    Overall, we have
    \begin{align*}
        \frac{f(x)}{A} 
        \geq \min\qty{\rme^{-2}, \frac{1-\rme^{-1/4}}{8}, \frac{1 - \rme^{-1}}{2}} M 
        > \frac{M}{40}.
    \end{align*}

    Finally, we evaluate $M$.
    If $B \geq (\rme - 1) A$, then $\log(1 + R) \geq 1$, 
    which implies $M \geq \min\{T, 1\} = 1$.
    Otherwise, we have $R < \rme - 1$, 
    which yields $\log(1 + R) \geq R/2 = \frac{B}{2A}$.
    Therefore, we have
    \begin{align*}
        f(x) \geq A \cdot \frac{\min\qty{1, \frac{B}{2A}}}{40}
        = \min\qty{\frac{A}{40}, \frac{B}{80}}, 
    \end{align*}
    for all $x\in[0, 1]$, which completes the proof.
\end{proof}

\section{More Details on Experimental Setups}\label[appendix]{sec:more-detail-experiment}

\paragraph{Computational resources}
We conducted all experiments on a cluster of eight NVIDIA A100 GPUs, each with 40 GB of VRAM.

\paragraph{License}
We use the pretrained model \texttt{imagenet64\_uncond\_100M\_1500K.pt}, which is provided under the MIT License.

\paragraph{Setups for baselines}
For all baseline methods, LoRA parameters are inserted only into attention layers, and the LoRA rank is set to $1$.
To select the learning rate, we first ran non-private SGD (3{,}000 steps) on $8$ datasets selected from the $100$ private $k$-NN sets and compared the average FID.
The results are shown in \cref{tab:lr-selection-baselines}.
Based on this preliminary experiment, we evaluated learning rates $10^2$ and $10^3$, which achieved the best performance among the candidates.
We also monitored the gradient norms during this preliminary experiment and set the clipping threshold to $10^{-2}$ accordingly.
For the number of training steps, we evaluated $100$ and $400$ steps.
In \cref{fig:fid-for-epsilons}, we report, for each $\epsilon$, the lowest FID among these combinations of learning rates and training steps.
The results for all hyperparameter combinations are available in \cref{tab:baseline-fid-scores}. 
We used a cosine learning-rate scheduler with a linear warmup over the first $10\%$ of training steps.
The batch size was set to $32$.

For PDA-DPMD, we set the coefficient of the public-data-based mirror map to $\alpha=0.5$.
See Section 5.1 of \citet{amid2022public} for the definition of $\alpha$.
For DP-SGD with public-data finetuning, we first finetuned the pretrained model on all $30{,}000$ public examples for $100$ epochs using Adam, a learning rate of $10^{-2}$, a batch size of $32$, and a cosine learning-rate scheduler with a linear warmup over the first $10\%$ of training steps.
In this public finetuning stage, LoRA parameters were inserted only into attention layers, and the LoRA rank was set to $8$.

\paragraph{Setups for DP-DeepSets}
For DP-DeepSets, we use the same LoRA configuration as the baselines: LoRA parameters are inserted only into attention layers, and the LoRA rank is set to $1$.
This results in $40{,}704$ LoRA parameters in total.
We split these parameters into $128$-dimensional tokens so that each token contains only parameters from the same position, resulting in $318$ tokens.
We set the dataset embedding dimension to $128$, matching the token dimension.
The hypernetwork uses a Transformer with hidden dimension $128$, depth $4$, MLP width $512$, and $4$ attention heads.
To prevent the generated LoRA parameters from having excessively large norms, the linear head is initialized from a normal distribution with standard deviation $10^{-2}$, rather than using the PyTorch default initialization.
We set the clipping scale for obtaining the noisy dataset embedding to $C=1$.
We train the hypernetwork for $200$ epochs using $30{,}000$ public datasets, each of which corresponds to a $k$-NN set of a CIFAR-10 data point with $k=128$.
For training, we use a learning rate of $10^{-3}$, a batch size of $4$, and a cosine learning-rate scheduler with a linear warmup over the first $10\%$ of training steps.

\section{Additional Experiment: Hypernetwork Training under Distribution Shift}

In the experiments above, the public and private datasets were constructed by splitting the same dataset into two subsets.
As an additional experiment, we consider the case where the two datasets are constructed from different open datasets.
That is, we investigate the out-of-distribution generalization ability of DP-DeepSets.

The private datasets are constructed from CIFAR-10 in the same way as before.
The public datasets are constructed from ImageNet64.
We downsample the $64 \times 64$ images to $32 \times 32$ images, matching the resolution of CIFAR-10.
As with CIFAR-10, for each data point, we construct a public dataset by taking its $k$-NN set in the Inception embedding space with $k=128$, yielding $1{,}281{,}167$ public datasets.
Using this collection of public datasets, we train DP-DeepSets for $5$ epochs.

We modify the DP-DeepSets configuration as follows: we set the number of Transformer layers to $6$,
increase the width of the feed-forward layers to $1024$, and increase the hidden dimension of the Transformer to $256$.
Because of the increased model size, we further reduce the initialization scale of the final layer (linear head) to $4 \times 10^{-3}$.
For the DP parameters, we keep $\delta=10^{-5}$ and conduct experiments with $\epsilon=1$ and $\epsilon=4$.
We set the clipping scale to $C=10$.

\Cref{tab:ood-fid-scores} compares the results of DP-DeepSets and DP-SGD.
For both values of $\epsilon$, DP-DeepSets achieves lower FID than DP-SGD.
This suggests that DP-DeepSets can serve as a better differentially private finetuning method than DP-SGD even when it is trained on public datasets whose distribution differs from that of the private dataset, especially when the private dataset is small.

\begin{table}[h]
    \centering
    \caption{FID scores for learning-rate selection in the preliminary non-private SGD experiment.}
    \label{tab:lr-selection-baselines}
    \scriptsize
    \setlength{\tabcolsep}{2pt}
    \resizebox{\linewidth}{!}{%
    \begin{tabular}{c|ccccccc}
        \hline
        Learning rate& $10^{4}$ & $10^{3}$ & $10^{2}$ & $10^{1}$ & $10^{0}$ & $10^{-1}$ & $10^{-2}$ \\
        \hline
        FID & $449.4 \pm 68.8$ & $114.0 \pm 15.1$ & $178.7 \pm 28.0$ & $240.2 \pm 16.3$ & $247.5 \pm 16.3$ & $246.0 \pm 13.6$ & $246.0 \pm 13.7$ \\
        \hline
    \end{tabular}
    }
\end{table}

\begin{table}[h]
    \centering
    \caption{FID scores for baseline methods across DP budgets, learning rates, and training steps.}
    \label{tab:baseline-fid-scores}
    \tiny
    \setlength{\tabcolsep}{1.5pt}
    \renewcommand{\arraystretch}{1.03}
    \setlength{\heavyrulewidth}{0.16em}
    \resizebox{\textwidth}{!}{%
    \begin{tabular}{lllcccccc}
        \specialrule{\heavyrulewidth}{0pt}{0ex}
        \multirow{2}{*}{\raisebox{-1.0ex}{Methods}} & \multirow{2}{*}{\raisebox{-2.0ex}[0pt][0pt]{\shortstack[l]{learning\\rate}}} & \multirow{2}{*}{\raisebox{-1.0ex}{Steps}} & \multicolumn{6}{c}{\raisebox{-0.45ex}{DP epsilon}} \\
        \cmidrule(lr){4-9}
        & & & $1024$ & $256$ & $64$ & $16$ & $4$ & $1$ \\
        \specialrule{\lightrulewidth}{0pt}{0.32ex}
        \specialrule{\lightrulewidth}{0pt}{0.25ex}
        PDA-DPMD & $10^{3}$ & $100$ & $243.6 \pm 30.3$ & $450.6 \pm 37.4$ & $450.7 \pm 42.7$ & $452.8 \pm 39.2$ & $453.9 \pm 39.9$ & $\textcolor{blue}{446.8 \pm 43.5}$ \\
        \cmidrule{3-9}
        & & $400$ & $452.2 \pm 36.5$ & $453.8 \pm 37.9$ & $452.4 \pm 39.3$ & $454.8 \pm 38.5$ & $453.4 \pm 39.7$ & $463.9 \pm 67.0$ \\
        \cmidrule{2-9}
        & $10^{2}$ & $100$ & $251.4 \pm 19.9$ & $251.7 \pm 20.0$ & $\textcolor{blue}{252.4 \pm 20.1}$ & $\textcolor{blue}{268.4 \pm 23.8}$ & $\textcolor{blue}{411.1 \pm 79.0}$ & $452.1 \pm 38.1$ \\
        \cmidrule{3-9}
        & & $400$ & $\textcolor{blue}{241.2 \pm 19.8}$ & $\textcolor{blue}{242.3 \pm 22.3}$ & $272.8 \pm 30.1$ & $439.9 \pm 63.8$ & $452.3 \pm 39.0$ & $452.8 \pm 38.4$ \\
        \cmidrule{1-9}
        PublicFT + DP-SGD & $10^{3}$ & $100$ & $\textcolor{blue}{221.4 \pm 34.0}$ & $297.6 \pm 45.3$ & $368.4 \pm 52.0$ & $434.7 \pm 44.8$ & $451.3 \pm 39.1$ & $448.7 \pm 38.3$ \\
        \cmidrule{3-9}
        & & $400$ & $290.8 \pm 47.3$ & $399.1 \pm 49.3$ & $451.5 \pm 38.2$ & $453.0 \pm 39.2$ & $448.8 \pm 38.7$ & $\textcolor{blue}{382.0 \pm 37.2}$ \\
        \cmidrule{2-9}
        & $10^{2}$ & $100$ & $230.2 \pm 29.0$ & $230.0 \pm 29.3$ & $\textcolor{blue}{230.1 \pm 29.7}$ & $\textcolor{blue}{236.5 \pm 31.9}$ & $\textcolor{blue}{344.0 \pm 46.9}$ & $430.0 \pm 35.6$ \\
        \cmidrule{3-9}
        & & $400$ & $226.4 \pm 29.1$ & $\textcolor{blue}{224.1 \pm 30.4}$ & $234.6 \pm 38.4$ & $259.4 \pm 38.5$ & $446.1 \pm 36.3$ & $446.3 \pm 38.1$ \\
        \specialrule{\heavyrulewidth}{0.7ex}{0ex}
        \multirow{2}{*}{\raisebox{-1.0ex}{Methods}} & \multirow{2}{*}{\raisebox{-2.0ex}[0pt][0pt]{\shortstack[l]{learning\\rate}}} & \multirow{2}{*}{\raisebox{-1.0ex}{Steps}} & \multicolumn{6}{c}{\raisebox{-0.45ex}{DP epsilon}} \\
        \cmidrule(lr){4-9}
        & & & $1024$ & $512$ & $256$ & $128$ & $64$ & $32$ \\
        \specialrule{\lightrulewidth}{0pt}{0.32ex}
        \specialrule{\lightrulewidth}{0pt}{0.25ex}
        DP-SGD & $10^{3}$ & $100$ & $233.4 \pm 31.9$ & $\textcolor{blue}{226.8 \pm 34.4}$ & $296.5 \pm 44.3$ & $342.5 \pm 48.0$ & $382.1 \pm 48.8$ & $418.7 \pm 45.5$ \\
        \cmidrule{3-9}
        & & $400$ & $287.4 \pm 44.5$ & $345.1 \pm 51.0$ & $396.7 \pm 47.2$ & $452.0 \pm 43.4$ & $454.8 \pm 39.3$ & $450.7 \pm 39.7$ \\
        \cmidrule{2-9}
        & $10^{2}$ & $100$ & $251.2 \pm 19.7$ & $251.3 \pm 20.0$ & $251.4 \pm 20.0$ & $251.6 \pm 19.7$ & $\textcolor{blue}{252.2 \pm 20.0}$ & $254.7 \pm 20.0$ \\
        \cmidrule{3-9}
        & & $400$ & $\textcolor{blue}{231.9 \pm 24.9}$ & $232.8 \pm 26.9$ & $\textcolor{blue}{235.7 \pm 30.2}$ & $\textcolor{blue}{245.2 \pm 34.9}$ & $261.0 \pm 42.0$ & $\textcolor{blue}{254.5 \pm 35.1}$ \\
        \cmidrule[\heavyrulewidth]{2-9}
        \noalign{\vskip -0.8ex}
        & \multirow{2}{*}{\raisebox{-2.0ex}[0pt][0pt]{\shortstack[l]{learning\\rate}}} & \multirow{2}{*}{\raisebox{-1.0ex}{Steps}} & \multicolumn{6}{c}{\raisebox{-0.45ex}{DP epsilon}} \\
        \noalign{\vskip -0.2ex}
        \cmidrule(lr){4-9}
        & & & $16$ & $8$ & $4$ & $2$ & $1$ & \\
        \noalign{\vskip -0.4ex}
        \cmidrule[\lightrulewidth]{2-9}
        \noalign{\vskip -1.3ex}
        \cmidrule[\lightrulewidth]{2-9}
        & $10^{3}$ & $100$ & $434.0 \pm 42.3$ & $450.7 \pm 38.7$ & $451.7 \pm 39.0$ & $449.5 \pm 39.1$ & $448.3 \pm 38.2$ & \\
        \cmidrule{3-9}
        & & $400$ & $452.1 \pm 38.7$ & $451.3 \pm 38.6$ & $449.0 \pm 38.6$ & $442.6 \pm 36.7$ & $\textcolor{blue}{382.2 \pm 37.1}$ & \\
        \cmidrule{2-9}
        & $10^{2}$ & $100$ & $\textcolor{blue}{266.3 \pm 20.3}$ & $\textcolor{blue}{302.2 \pm 30.5}$ & $\textcolor{blue}{348.7 \pm 47.6}$ & $\textcolor{blue}{377.7 \pm 39.4}$ & $435.9 \pm 35.6$ & \\
        \cmidrule{3-9}
        & & $400$ & $271.3 \pm 45.9$ & $386.7 \pm 52.0$ & $439.2 \pm 35.6$ & $445.4 \pm 37.1$ & $448.9 \pm 39.8$ & \\
        \bottomrule
    \end{tabular}
    }
\end{table}

\begin{table}[h]
    \centering
    \caption{FID scores under distribution shift between public and private datasets.}
    \label{tab:ood-fid-scores}
    \begin{tabular}{c|cc}
        \hline
        DP epsilon & 1 & 4 \\
        \hline
        Ours   & {\bf 228.8 $\pm$ 55.8} & {\bf 224.0 $\pm$ 55.3} \\
        DP-SGD & 382.2 $\pm$ 37.1 & 348.7 $\pm$ 47.6 \\
        \hline
    \end{tabular}
\end{table}



\end{document}